%% file: cvpr_main.tex
\documentclass[11pt]{article}
\usepackage{bbm}

\input{math_commands.tex}

\usepackage{amsthm}
\usepackage{times}
\usepackage{epsfig}
\usepackage{amsmath}
\usepackage{amssymb}
\usepackage{url}
\usepackage{algorithm}  
\usepackage{algorithmicx}  
\usepackage{algpseudocode} 
\usepackage{booktabs}       % professional-quality tables
\usepackage{multirow}
\usepackage{graphicx}
\usepackage{wrapfig}
\usepackage{subfigure}
\usepackage{multirow}

\newcommand{\etal}{\textit{et al.}}

\usepackage[pagebackref=true,breaklinks=true,colorlinks,bookmarks=false]{hyperref}

\begin{document}

%%%%%%%%% TITLE
\title{Provable Defense against Privacy Leakage in Federated Learning from Representation Perspective}

\author{Jingwei Sun, Ang Li, Binghui Wang, Huanrui Yang, Hai Li, Yiran Chen\\
\textit{Department of Electrical and Computer Engineering, Duke University}\\
\{jingwei.sun, ang.li630, binghui.wang, huanrui.yang, hai.li, yiran.chen\}@duke.edu}

\maketitle

%%%%%%%%% ABSTRACT
\begin{abstract}
   Federated learning (FL) is a popular distributed learning framework that can reduce privacy risks by not explicitly sharing private data. %It can reduce privacy risks. 
   However, recent works demonstrated that sharing model updates makes FL vulnerable to inference attack. In this work, we show our key observation that the data representation leakage from gradients is the essential cause of privacy leakage in FL. We also provide an analysis of this observation to explain how the data presentation is leaked. 
   Based on this observation, we propose a defense against model inversion attack in FL. The key idea of our defense is learning to 
   % optimize a minimal perturbation on 
   perturb data representation such that the quality of the reconstructed data is severely degraded, while FL performance is maintained. In addition, we derive certified robustness guarantee to FL and convergence guarantee to FedAvg, 
   after applying our defense. 
   % theoretical results on both certified defense bound and the convergence guarantee 
   % after applying our defense to FL. 
   %is also theoretically guaranteed. 
   To evaluate our defense, we conduct experiments on MNIST and CIFAR10 for defending against the DLG attack and GS attack. Without sacrificing accuracy, the results demonstrate that our proposed defense can increase the mean squared error between the reconstructed data and the raw data by as much as more than 160$\times$ for both DLG attack and GS attack, compared with baseline defense methods. The privacy of the FL system is significantly improved. 
\end{abstract}

%%%%%%%%% BODY TEXT
\input{Introduction}
\input{related_work}

\input{leak}

\input{defense}
\input{theoretical_guarantee}
\input{experiments}
\input{conclusion}

\bibliographystyle{IEEEtran}
\bibliography{reference}

\clearpage

\appendix

\input{appendix_exp}

\input{appendix_theorem}

\end{document}

%% file: math_commands.tex
\usepackage{amsmath,amsfonts,bm}

\def\figref#1{figure~\ref{#1}}
\def\eqref#1{equation~\ref{#1}}
\def\1{\bm{1}}

\def\vb{{\bm{b}}}

\def\vr{{\bm{r}}}

\def\vw{{\bm{w}}}

\def\vy{{\bm{y}}}

\def\mV{{\bm{V}}}
\def\mW{{\bm{W}}}
\def\mX{{\bm{X}}}

\DeclareMathAlphabet{\mathsfit}{\encodingdefault}{\sfdefault}{m}{sl}
\SetMathAlphabet{\mathsfit}{bold}{\encodingdefault}{\sfdefault}{bx}{n}

\def\sB{{\mathbb{B}}}

\def\sN{{\mathbb{N}}}

\def\sR{{\mathbb{R}}}
\def\sS{{\mathbb{S}}}

%% file: Introduction.tex
\section{Introduction}\label{sec:Introduction}

Federated learning (FL)~\cite{mcmahan2017communication} is a popular distributed learning approach that enables a number of devices to train a shared model in a federated fashion without transferring their local data. A central server coordinates the FL process, where each participating device communicates only the model parameters on the central server while keeping local data private. 
Thus, FL becomes a natural choice for developing mobile deep learning applications, such as next-word prediction~\cite{hard2018federated}, emoji prediction~\cite{ramaswamy2019federated}, etc.
%However, the data  is often non-IID (identically and independently distributed) across a massive number of devices in FL. For example, the input words to a virtual keyboard~\cite{hard2018federated} on a mobile phone are often user and context dependent, hence, these text data is highly heterogeneous across devices. 

%Another huge contribution FL brings is that it achieves good performance under data non-iid settings, which is common in real life. FL is widely used in practical applications where data on remote clients contains private information and reveals personalized distributions, such as next-word-prediction on smarts phones and training a model using patients’ medical data from different hospitals.

Privacy preserving is the major motivation for proposing FL. However, recent works demonstrated that sharing model updates or gradients also makes FL vulnerable to inference attack, e.g., \textit{property inference attack}~\cite{melis2019exploiting} and \textit{model inversion attack}~\cite{fredrikson2015model, zhu2019deep, geiping2020inverting, wang2019beyond}. 
Here property inference attack infers sensitive properties of training data using the model updates and model inversion attack reconstructs training data using model gradients. 
However, the essential causes of such privacy leakages have not been thoroughly investigated or explained.
Some defense strategies have been presented to prevent the privacy leakage and can be categorized into three types:
%. These methods are build on three major techniques: 
\textit{differential privacy}~\cite{pathak2010multiparty,shokri2015privacy,hamm2016learning,mcmahan2018learning,geyer2017differentially}, \textit{secure multi-party computation}~\cite{danner2015fully,mohassel2017secureml,bonawitz2017practical,melis2019exploiting}, and \textit{data compression}~\cite{zhu2019deep}. 
But these defensive approaches incur either significant computational overheads or unignorable accuracy loss. 
The reason is that existing
% the aforementioned 
defenses are not specifically designed for the privacy leakage from the communicated local updates.
%due to the lack of 
% theoretical foundation for  
%understanding the privacy leakage. 
The privacy issues 
% continue 
seriously 
% hindering 
hinder the development and deployment of FL. 
There is an urgent necessity to unveil the essential cause of privacy leakage 
% to incur the privacy leakage, 
such that we can develop 
% efficient and 
effective defenses to tackle the privacy issue of FL.
%strategy for 
%enhancing the security of FL.

In this work, we assume the server in FL is malicious, and it aims to reconstruct the private training data 
% privacy information of training data 
from 
% participating 
devices.
% We show 
Our key observation is: \textbf{the class-wise data representations of each device's data are embedded in shared local model updates, and such data representations can be inferred to perform model inversion attacks. Therefore, the information can be severely leaked through the model updates.} 
In particular, we provide an analysis to reveal how the data representations, e.g., in the fully connected (FC) layer, are embedded in the model updates. 
We then propose an algorithm to infer class-wise data representation to perform model inversion attacks.
Our empirical study demonstrates that the correlation between the inferred data representations using our algorithm and the real data representations is as high as 0.99 during local training, and thus prove that the representations leakage is the essential cause behind existing attacks. 
Note that the data is often non-IID (identically and independently distributed) across 
% a massive number of 
the devices in FL. We also show that the non-IID characteristic aggravates the representation leakage. % embedded in the local updates. 

Based on our observation of the representation leakage from local updates, we design a defense strategy. Specifically, we present an algorithm to generate a perturbation added to the data representation, such that: 
1) the perturbed data representations are as similar as possible to the true data representations to maintain the FL performance; 
and 2) the reconstructed data using the perturbed data representations are as dissimilar as possible to the raw data. 
Importantly, we also derive 
%theoretical defense bound and convergence guarantee with 
certified robustness guarantee to FL   
and 
convergence guarantee to FedAvg, a popular FL algorithm, when applying our defense. 
%provide the theoretical guarantee for the defense and the convergence 
% when applying our proposed defense to FL. 
To evaluate the effectiveness of our defense, we conduct experiments on MNIST and CIFAR10 for defending against the DLG attack~\cite{zhu2019deep} and GS attack~\cite{geiping2020inverting}. The results demonstrate that without sacrificing accuracy, our proposed defense can increase squared error (MSE) mean between the reconstructed data and the raw data for both DLG attack and GS attack by as many as more than 160$\times$, compared with baseline defense methods. %showing a significantly stronger privacy guarantee.
The privacy of the FL system is significantly improved.
 
Our key contributions are summarized as follows:
 \begin{itemize}
    \item To the best of our knowledge, this is the first work to explicitly reveal that data representations embedded in the model updates is the essential cause of leaking private information from the communicated local updates in FL. In addition, we  %provide an analysis of the observed data representation leakages to 
    develop an algorithm to effectively reconstruct the data from the local updates. 
    %extract class-wise data representations from local updates to perform .
    \item 
    % Based on our observation, 
    We develop an effective defense by perturbing 
    % seeking an optimized perturbation to 
    data representations. We also derive certified robustness guarantee to FL and convergence guarantee to FedAvg, when applying our defense.
    \item We empirically evaluate our defense on MNIST and CIFAR10 against DLG  and GS attacks. The results show our defense can offer a significantly stronger privacy guarantee without sacrificing accuracy.
\end{itemize}

%% file: related_work.tex
\section{Related work}\label{sec:Related work}
\paragraph{Privacy Leakage in Distributed Learning.} There exist several adversarial goals to infer private information: \textit{data reconstruction}, \textit{class representative inference}, \textit{membership inference}, and \textit{attribute inference}.
Data reconstruction aims to recover training samples that are used by participating clients. The quality of the reconstructed samples can be assessed by comparing the similarity with the original data. Recently, Zhu \etal~\cite{zhu2019deep} present an algorithm named \textit{DLG} to reconstruct training samples by optimizing the input to generate the same gradients for a particular client. Following up DLG, \textit{iDLG}~\cite{zhao2020idlg} is proposed to improve the efficiency and accuracy of DLG. Aono \etal~\cite{aono2017privacy} also show that an honest-but-curious server can partially reconstruct clients' training inputs using their local updates. However, such an attack is applicable only when the batch consists of a single sample. Wang \etal~\cite{wang2019beyond} present a reconstruction attack by incorporating a generative adversarial network (GAN) with a multi-task discriminator. But this method is only applicable to scenarios where data is mostly homogeneous across clients and auxiliary dataset is available. 
Several approaches have been proposed to infer class features or class representatives. Hitaj \etal~\cite{hitaj2017deep} demonstrate that an adversarial participant in the collaborative learning can utilize GANs to construct class representatives. However, this technique is evaluated only when all samples of the same class are virtually similar (e.g., handwritten digits, faces, etc.).
Membership inference attack (MIA) is performed to accurately determine whether a given sample has been used for the training. This type of attack is first proposed by Shokri\etal~\cite{shokri2017membership}, and it can be applied to any types of machine learning models even under black-box settings. 
Sablayrolles \etal~\cite{sablayrolles2019white} propose an optimal strategy for MIA under the assumption that model parameters conform to certain distributions. Nasr \etal~\cite{nasr2019comprehensive} extend MIA to federated learning for quantifying the privacy leakage in the distributed setting.
Attribute inference attack tries to identify some sensitive attributes of training data. There exist several techniques to perform this type of attack~\cite{ateniese2015hacking,fredrikson2014privacy,fredrikson2015model,hitaj2017deep}. Fredrikson \etal~\cite{fredrikson2014privacy} propose a method to reveal genomic information of patients using model outputs and other non-sensitive attributes. More recently, Melis \etal~\cite{melis2019exploiting} demonstrate that an adversarial client can infer attributes that hold only for a subset of the training data based on the exchanged model updates in federated learning.

\begin{figure*}[t]
\centering
     \includegraphics[scale=0.45]{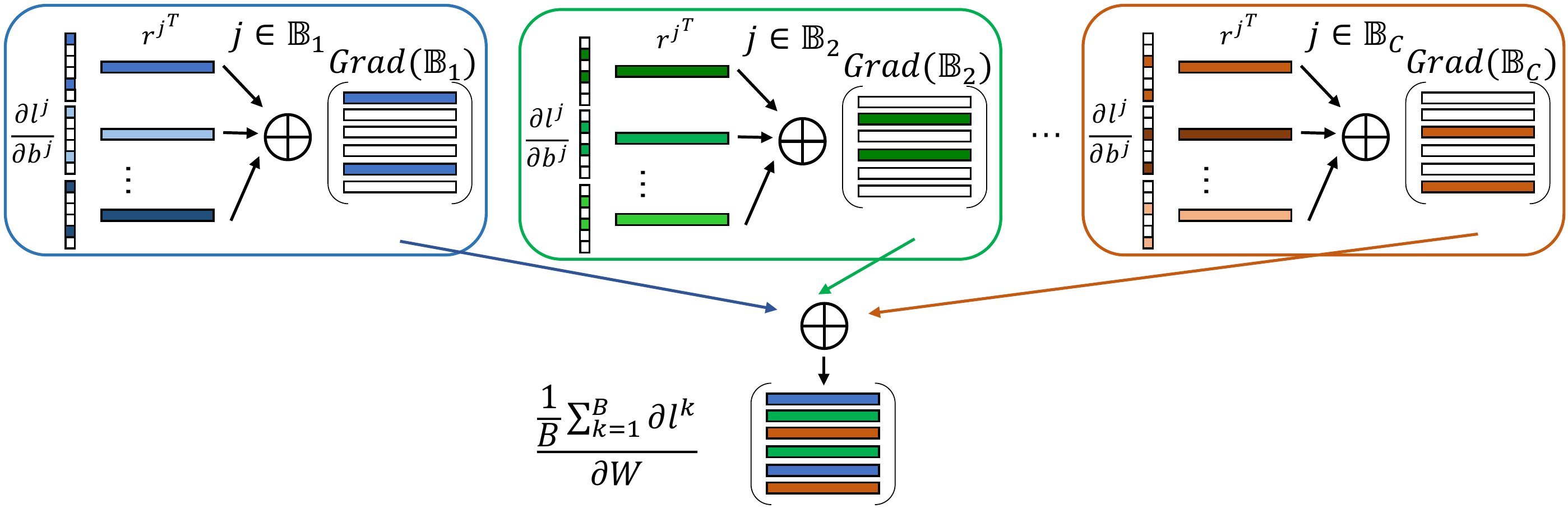}
\caption{Illustration of the gradient updates of class-wise data in a batch.}
\label{fig:GradB_sum}
\end{figure*}

\paragraph{Privacy-preserving Distributed Learning.}
Existing privacy-preserving distributed learning methods can be categorized into three types: \textit{differential privacy} (DP), \textit{secure multi-party computation} (MPC), and \textit{data compression}. 
Pathak \etal~\cite{pathak2010multiparty} present a distributed learning method to compose a deferentially private global model by aggregating locally trained models.
Shokri \etal~\cite{shokri2015privacy} propose a collaborative learning method where the sparse vector is adopted to achieve DP.
Hamm \etal~\cite{hamm2016learning} design a distributed learning approach to train a  deferentially private global model via transferring the knowledge of the local model ensemble. Recently, participant-level deferentially private federated learning are proposed~\cite{mcmahan2018learning,geyer2017differentially} via injecting Gaussian noise to local updates. 
However, these DP-based methods require a large number of participants in the training to converge and realize a desirable privacy-performance tradeoff. 
In addition, MPC has also been applied to develop privacy-preserving machine learning in a distributed fashion. For example, Danner \etal~\cite{danner2015fully} propose a secure sum protocol using a tree topology. Another example of the MPC-based approach is SecureML~\cite{mohassel2017secureml}, where participants distribute their private data among two non-colluding servers, and then the two servers use MPC to train a global model using the participants' encrypted joint data. Bonawitz \etal~\cite{bonawitz2017practical} propose a secure multi-party aggregation method for FL, 
% federated learning, 
where participants are required to encrypt their local updates such that the central server can only recover the aggregation of the updates. However, these MPC-based approaches will incur unneglectable computational overhead. It is even worse that attackers can still successfully infer private information even if the adversary only observes the aggregated updates~\cite{melis2019exploiting}. 
Furthermore, Zhu \etal~\cite{zhu2019deep} show applying gradient compression and sparsification can help defend against privacy leakage from shared local updates. However, such approaches require a high compression rate to achieve a desirable defensive performance. In Section \ref{sec:experiments}, given the same compression rate, we show that our proposed method can achieve better defense and inference performance than that of the gradient compression approach.

%% file: leak.tex
%\section{Essential Cause for Information Leakage in FL}
\section{Essential Cause of Privacy Leakage in FL}
\label{sec:theory}

Existing works~\cite{zhu2019deep,zhao2020idlg,aono2017privacy,wang2019beyond} demonstrate that information leakage is from communicated model updates between the devices and server during FL training. However, they do not provide a thorough explanation. 
To understand the essential cause of information leakage in FL, 
we analyze the privacy leakage in FL. 
Our key observation is that privacy leakage is essentially caused by the data representations embedded in the model updates. 
% In addition, we show that our proposed theoretical basis can explicitly explain the success of existing attacks.

%\vspace{-2mm}
\subsection{Representation Leakage in FL}\label{sec:leakage}
\noindent {\bf Problem setup.} 
In FL, there are multiple client devices and a center server. The server coordinates the FL process,  where  each  participating  device  communicates  only the  model  parameters  with  the server  while  keeping their local  data  private.
%Following the training process of FL described in \cite{mcmahan2017communication}, we assume multiple 
% $D (D\ge2)$ devices participate in federated training. 
%For one device, its data show strong personalization and involve privacy information. 
%In terms of privacy leakage, 
We assume the server is malicious and it only has access to the devices' model parameters. 
The server's purpose is to infer the devices' data through the devices' model parameters. 
% points showing the training data privacy information on that device. 
%The server has only access to the devices' model parameters 
%the devices uploaded and invert the data points through the model parameters and the parameter updates.

\noindent{\bf Key observations on representation leakage in FL: Data representations are less entangled.}
%However, the essential cause leading to such information leakages is not thoroughly investigated or explained. 
%To understand the essential cause of information leakage in FL, 
For simplicity, we use the fully connected (FC) layer as an instance and analyze how data representation is leaked in FL. 
%its model is updated during training.  
%first analyze  how a fully connected (FC) layer is updated in a neural network during training. 
We note that  such an analysis can be naturally extended to other types of layers. Specifically, we denote a FC layer as $\vb = \mW \vr$, where $\vr$ is the input to the FC layer (i.e., the learnt data representation by previous layers), $\mW$ is the weight matrix, and $\vb$ is the output.
Then, given a training batch $\sB$, the gradient of the loss $l$ with respect to $\mW$ is:% as follows: 
%can be formulated as Equ.~(\ref{eq:fc_grad}):
{\small
\begin{equation}
\frac{\frac{1}{|\sB|} \sum_{i=1}^{|\sB|} \partial l^{i}} {\partial \mW} = 
\frac{1}{|\sB|} \sum_{i=1}^{|\sB|} \frac{\partial l^{i}}{\partial \vb^{i}} \frac{\partial \vb}{\partial \mW} =
\frac{1}{|\sB|} \sum_{i=1}^{|\sB|} \frac{\partial l^{i}}{\partial \vb^{i}}({\vr^{i}})^{\mathrm{T}},
\label{eq:fc_grad}
\end{equation} 
}%
where $l^i$, $\vr^i$, and $\vb^i$ are the loss corresponding to the $i^{th}$ sample, the input, and the output of the FC layer in this batch, respectively.
%, $\vr^i$ and $\vb^i$ denote the input and output  of this layer regarding the $i^{th}$ sample in this batch. 
We observe that the gradient for a particular sample is the product of a column vector $\frac{\partial l^{i}}{\partial \vb^{i}}$ and a row vector $({\vr^{i}})^{\mathrm{T}}$.
%\vspace{-4mm}
%\paragraph{Less Entangled Data Representations in FL.}
Suppose the training data has $C$ labels. 
We can split the batch $\sB$ into $C$ sets, i.e., $\sB = \{\sB_0, \sB_1, \dots, \sB_C\}$, where $\sB_k$ denotes the data samples with the $k$-th label.
%the $k^{th}$ class  and $C$ is the total number of data classes in this batch. 
Then, Eq. \ref{eq:fc_grad} can be rewritten as:
{\small
\begin{equation}
\frac{\frac{1}{|\sB|} \sum_{k=1}^{|\sB|} \partial l^{k}} {\partial \mW} 
%\frac{1}{|\sB|} \frac{\sum_{k=1}^{|\sB|} \partial l^{k}}{\partial \mW}
= \sum_{i=1}^{C} \left( \frac{1}{\left|\sB_i\right|} \sum_{j \in \sB_i} \frac{\partial l^{j}}{\partial \vb^{j}} ({\vr^{j}})^{\mathrm{T}} \right) 
\stackrel{\Delta}{=} \sum_{i=1}^{C} Grad(\sB_i),
\label{eq:fc_grad_B}
\end{equation}
}%
where 
%$\left|\sB_i\right|$ denotes the volume of data with the $i$th class in this batch.
$Grad(\sB_i)$ represents the gradient with respect to the data samples in $\sB_i$. 
Figure~\ref{fig:GradB_sum} illustrates the gradient updates for a batch data in a FC layer.  
%As shown in Figure~\ref{fig:GradB_sum},  a FC layer is updated using Equ. (\ref{eq:fc_grad_B}). 
We observe that for data coming from different classes, the corresponding data representations tend to be embedded in different rows of gradients. 
If the number of classes is large in a batch, which is common in centralized training, the representations of different classes will be entangled in the gradients of this whole batch.
%\paragraph{Less Entangled Data Representations in FL}
In contrast to 
% Different from 
centralized training, the local data often covers a small number of tasks on a participating device in FL. 
% Consequently, 
Thus, the number of data classes $C$ within one training batch may be very small compared to that of the centralized training. In this case, the entanglement of data representations from different classes can be significantly reduced. 
Such a low entanglement of data representations allows us to 
explicitly reconstruct
the input data of each class from the gradients, because we can (almost) precisely locate the rows of data representations in the gradients.

Note that in the above analysis, we only consider a single batch during the FL training. In practice, FL is often trained with multiple batches. In this case, the data representations of different classes could be entangled, especially when the number of batches is large. 
However, in practical FL applications, the devices often have insufficient data. During FL training, the numbers of batches and local training epochs of each device are both small. %has a small number of batches and a small number of local training epochs. 
In this case, the data representations could still be less entangled across classes through inspecting the gradient updates in Eq. \ref{eq:fc_grad_B}.

\subsection{Inferring Class-wise Data Representations }\label{sec:inferring representations}

\begin{figure}[t]
\centering
     \includegraphics[scale=0.4]{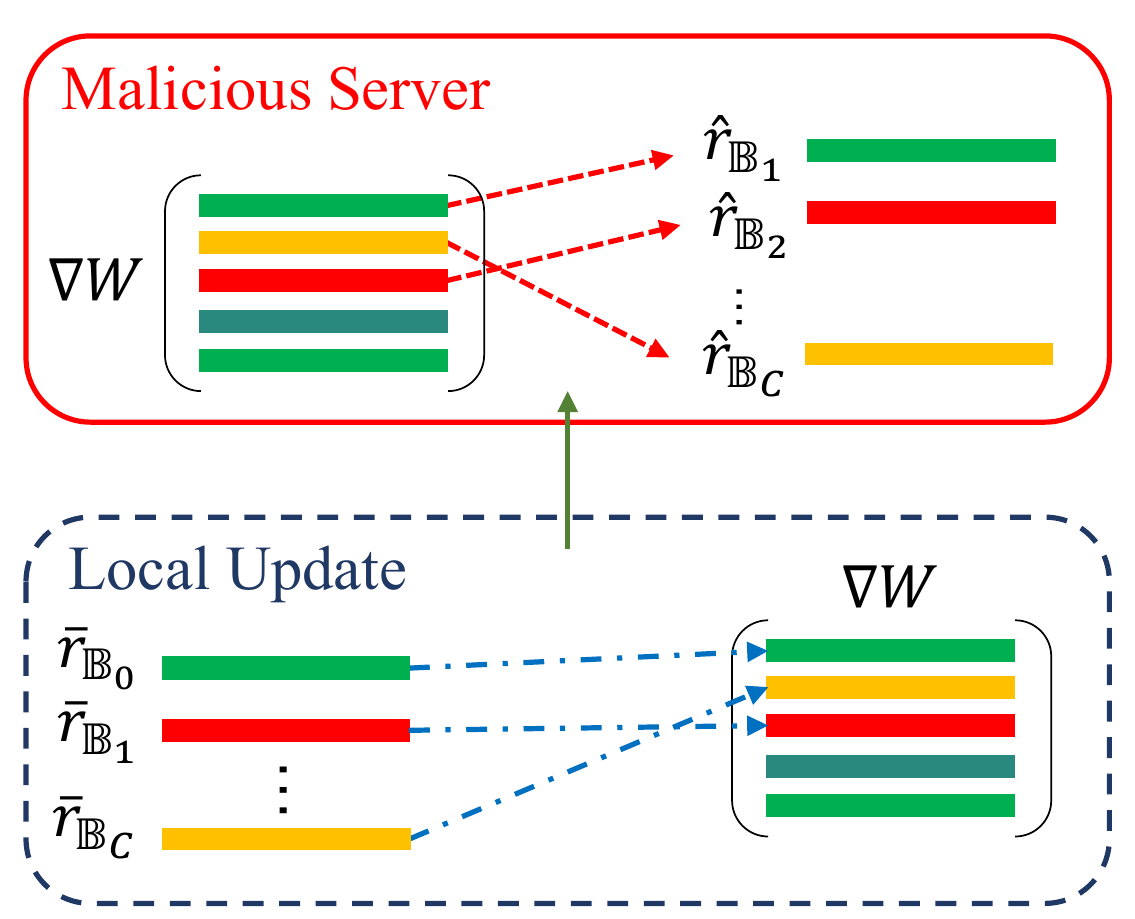}
\caption{Illustration of our representation inference algorithm.}
%\caption{The malicious server can extract representation $\hat{\vr}_{\sB_i}^{\mathrm{T}}$ corresponding to $\overline{\vr}_{\sB_i}^{\mathrm{T}}$, which is the averaged real training data representations of class $i$ on this device.}

\label{fig:feature_extractor_algorithm}
\end{figure}

% \begin{figure*}[!t]
%     \centering
%     \subfigure[FC1]{
%     \centering
%     \includegraphics[scale=0.32]{images/cifar10_fc1.eps}}
%     \subfigure[FC2]{
%     \centering
%     \includegraphics[scale=0.32]{images/cifar10_fc2.eps}}
%     \subfigure[FC3]{
%     \centering
%     \includegraphics[scale=0.32]{images/cifar10_fc3.eps}
%     }
%     \caption{$cor$ in different layers during communication rounds for CIFAR10. Under the same local training setting, the correlation between $\hat{\vr}_{\sB_i}^{\mathrm{T}}$ and $\overline{\vr}_{\sB_i}^{\mathrm{T}}$ is lower in shallower layers (layers closer to the input layer). This is because the data representations in the shallower layers are more general[], thus the structural specification of $Grad_{\sB}(i)$ in Equation \ref{eq:fc_grad_B} according to different data classes will be lower. As a result, the representation extraction in shallower layers will be harder compared to deeper layers.}
    
%     \label{fig:cifar10_fc}
% \end{figure*}

% Based on the aforementioned analysis, 
We develop an algorithm to identify the training classes and infer the data representations of each  class 
% $\hat{\vr}_{\sB_i}^{\mathrm{T}}$  
embedded in each FC layer from the model updates. The representation inference algorithm is in a back propagation fashion.
Specifically, we first identify the classes using the gradients of the last ($L-th$) layer. We denote the gradients of the $L$-th layer as $\nabla \mW^{L}$. We notice that the gradient vector ${\nabla \mW_i^{L}}$, which is the $i$th row in $\nabla \mW^{L}$, shows significantly larger magnitudes than gradient vectors in other rows if the data from $i$-th class are involved in training. Then, we can infer the data representation of the $i$-th class in the last layer,  because it linearly scales ${\nabla \mW_i^{L}}$.
% is just the linearly scaled ${\nabla \mW^{L}}_{i}$. 
If data representation of the $i$-th class in layer ${\mW^{L}}$ is inferred, we can use their element values to 
% its element values can assist us to 
identify the corresponding row from the $(L-1)$-th layer's gradients, i.e., ${\mW^{L-1}}$, which embeds the data representation of the $i$-th class in the $(L-1)$-th layer. 
% We can iteratively perform such representation inference until reaching the input layer. 
In this way, we can iteratively infer data representations of the $i$-th class in all FC layers. The inference process for one FC layer is illustrated in \figref{fig:feature_extractor_algorithm} and the details of our representation inference algorithm are described in Appendix~\ref{sec:method}. 

% \noindent{\bf Experimental setup.} 
We conduct experiments on CIFAR10 \cite{krizhevsky2009learning} to evaluate the effectiveness of our algorithm.
We consider the practical non-IID settings in FL, and follow the \textbf{2-class \& balanced} configuration in \cite{li2020lotteryfl} to construct non-IID datasets: 100 devices in total and 10 devices are randomly sampled to participate in training in each communication round. Each device holds 2 classes of data and each class has 20 samples. 

As 
% analysed in Section \ref{sec:leakage}, 
local training configurations can affect the performance of inferred representation. 
%is affected by the local training configurations. 
In this experiment, we vary the number of local training epochs $E \in \{ 1, 5, 10\}$ and local batch size $B \in \{ 8, 16, 32\}$. We adopt SGD as the optimizer with a learning rate $\eta=0.01$. 
The model architecture is shown in Appendix~\ref{sec:setup}.
% To demonstrate that non-IID data aggravates representation leakage, 
We also consider a baseline in the IID setting, where we set $E$ to be 1 and $B$ to be 32.

\begin{table}[t]
    \centering
    \footnotesize
    \caption{Average $cor$ across 200 communication rounds for different layers under different settings.}
        \label{tb:cor_average}
        \begin{tabular}{l || c c c}
            \toprule
            Local Training Configurations & FC1 & FC2 & FC3 \\
            \hline
            E=1, B=32 & 0.98 & 0.99 & 0.99 \\
            \hline
            E=5, B=32 & 0.82 & 0.90 & 0.92 \\
            \hline
            E=10, B=32 & 0.70 & 0.78 & 0.82 \\
            \hline
            E=1, B=16 & 0.82 & 0.93 & 0.99 \\
            \hline
            E=1, B=8 & 0.85 & 0.89 & 0.92 \\
            \hline
            E=1, B=32 (IID) & 0.48 & 0.31 & 0.18 \\
            \bottomrule
        \end{tabular}
\end{table}

We use the correlation coefficient $cor$ between the true representation $\overline{\vr}_{\sB_i}^{\mathrm{T}}$ and our inferred $\hat{\vr}_{\sB_i}^{\mathrm{T}}$ to quantify the effectiveness of our proposed algorithm. We calculate $cor$ for each class on each 
participating device. 
We extract data representations from all the FC layers in each of 200 communication rounds between the devices and the server, and the average $cor$ across all communication 
rounds and devices is shown in Table \ref{tb:cor_average}. 
As Table \ref{tb:cor_average} presents,  
% We observe that 
the correlation $cor$ 
% between our inferred representations $\hat{\vr}_{\sB_i}^{\mathrm{T}}$ and the true representations $\overline{\vr}_{\sB_i}^{\mathrm{T}}$ 
is as high as 0.99, indicating a serious representation leakage in FL. 
$cor$ decreases with $B$ or a larger number of batches in one epoch and increases as $E$ goes lower, which validate our claim in Section~\ref{sec:leakage}.
However, $cor$ is still higher than 0.8 in almost all cases. 
We note that $cor$ is much lower in the IID-setting. 
This is because each device has more classes of data for training than those in non-IID setting, making the representations entangled. 
Our results validate that the practical non-IID setting in FL dramatically worsens the representation leakage.

\begin{figure}[t]
\centering
     \includegraphics[scale=0.25]{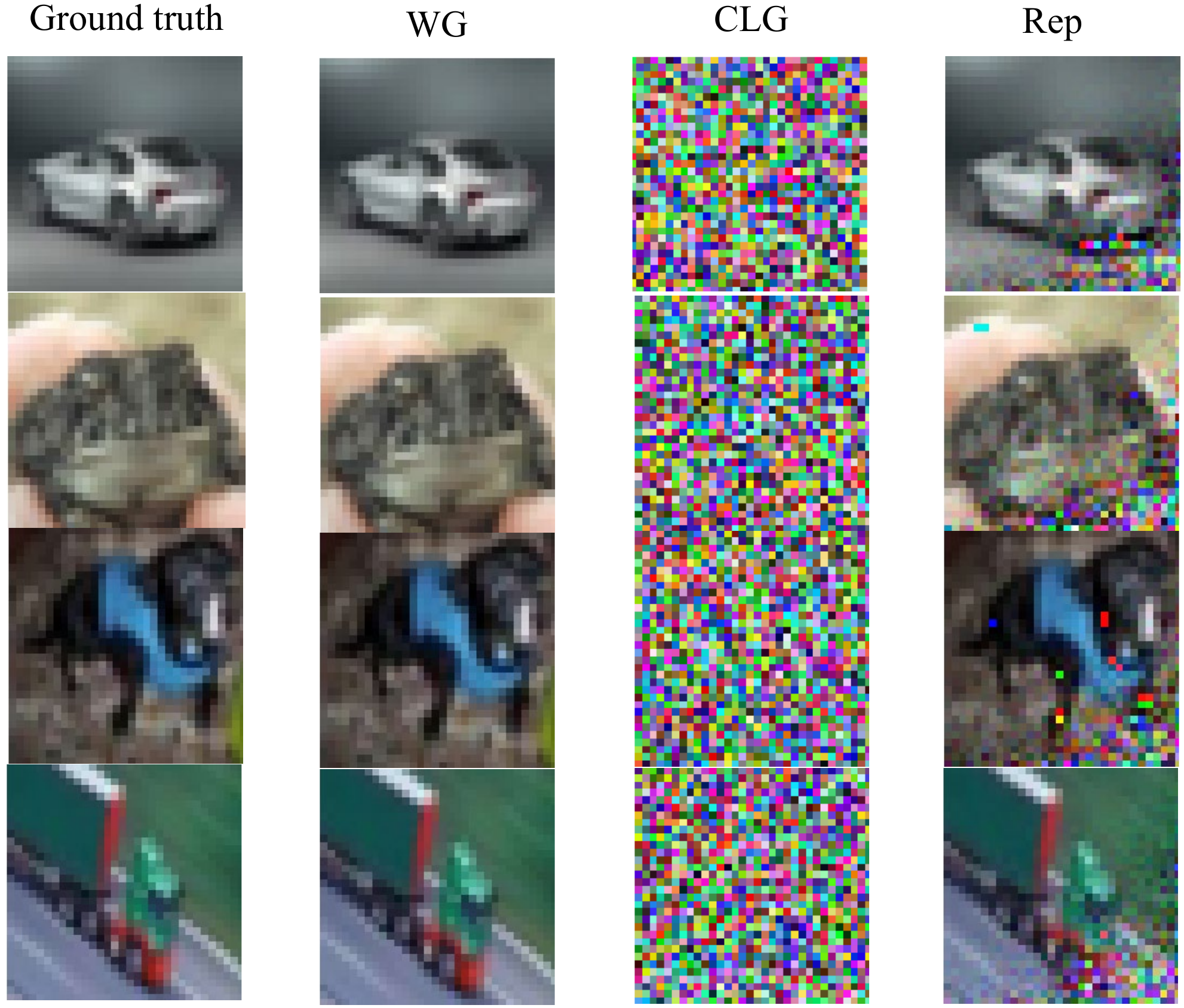}
\caption{DLG attack results utilizing different parts of gradients.}
\label{fig:prove_representation_essential}
\end{figure}

\subsection{Unveiling Representation Leakage}\label{sec:unveiling representation leakage}
%\subsection{Unveiling Representation Leakage under Existing Attacks}
%Based on previous analysis and experiment, we have illustrated how seriously the data representations leak in FL. 
%Previous works show that malicious server~\cite{zhu2019deep, geiping2020inverting} in FL often utilizes the whole uploaded gradients or model updates to leak privacy information. 
In this section, we investigate whether the representation leakage is the essential cause of information leakage in FL. 
Particularly, we conduct experiments on CIFAR10 to reconstruct the input based on the existing DLG attack~\cite{zhu2019deep}. DLG attack requires the gradient information, and we consider three different portions of the gradients: \textit{the whole model gradients (WG)}, \textit{the gradients of convolutional layers only (CLG)}, and \textit{inferred representations using our method (Rep)}. The experiment settings are presented in Appendix~\ref{sec:setup}.
% We conduct experiments on CIFAR10 based on DLG~\cite{zhu2019deep} attack  but utilize different portions of the gradients: \textit{the whole model gradients},  \textit{extracted representations using our method}, and \textit{the gradients of convolutional layers only}.}. 
As \figref{fig:prove_representation_essential} shows, only utilizing gradients of convolutional layers cannot successfully reconstruct the input data, but using the representation inferred by our method can reconstruct the input data as effectively as utilizing the whole gradients in terms of visual quality. This result validates that representation leakage is the essential cause of privacy leakage in FL.

%% file: defense.tex
\section{Defense Design}\label{sec:defense}

\subsection{Defense Formulation}
%\subsection{Motivation}
Our aforementioned observation shows that the privacy leakage in FL
%, i.e., the input can be reconstructed, 
mainly comes from the representation leakage (e.g., in the FC layer).  
In this section, we propose a defense against such privacy leakage. 
% Therefore, to defend against such information leakage, 
In particular, we propose to perturb the data representation in a single layer (e.g., a FC layer), which we call \emph{the defended layer}, to satisfy the following two goals: 
% , we aim to design a defense strategy that can control the data representation learning during local model updates.
% %, in order to obfuscate the . before data representations are embedded into gradients. 
% Our defense needs to satisfy the following two requirements: 
% %should have two basic satisfactions. 
\begin{itemize}
%For FL performance and convergence guarantee, the gap between tuned representations and real representations should not be too far. 
\item Goal 1: To reduce the privacy information leakage, the reconstructed input through the perturbed data representations and the raw input should be dissimilar.   
\item Goal 2: To maintain the FL performance, the perturbed 
data representation and the true data representations without perturbation should be similar. 

\end{itemize}

% fake data points inverted through the tuned representations should be different from the real data points as far as possible. Satisfying these two targets, our defense strategy would bring challenges to the inference attacks in FL , which utilize representations in gradients implicitly, while sacrificing performance as less as possible.

Let $r$ and $r'$ represent the clean data representation and perturbed data representation on the defended layer, respectively. 
We also define $X$ and $X'$ as the raw input and the reconstructed input via the perturbed data representation. 
To satisfy Goal 1, we require that the distance between $X$ and $X'$, in terms of $L_p$ norm, should be as large as possible;
To satisfy Goal 2, we require that the distance between $r$ and $r'$, in terms of $L_q$ norm, should be bounded. 
Formally, we have the following constrained objective function with respect to $r'$:
\begin{align}
%\label{eq:r_bound}
& \textrm{\bf Achieving Goal 1: } \max\limits_{r'} || X-X' ||_p,  \label{max_diff} \\
& \textrm{\bf Achieving Goal 2: } \, \, \textrm{s.t., } \, \, || r-r' ||_q \leq \epsilon, \label{min_rep}
\end{align}
where $\epsilon$ is a predeminted threshold. Note that $X'$ depends on $r'$. 
Next, we design a solution to obtain $r'$ and derive the certified robustness. %defense bound.  
%two objectives 
% For a real representation $r$ and its tuned result $r'$, their distance should be bounded for FL performance and convergence guarantee which can be formulated as {\equref{eq:r_bound}}. For the real data point $X_{\leftarrow r}$ and the inverted data point corresponding to the tuned representation $X'$, the gap between them should be maximized, which can be formulated as {\equref{eq:max_x_error}}.

% \begin{equation}
% || r-r' ||_q \leq \epsilon
% \label{eq:r_bound}
% \end{equation}

% \begin{equation}
% r'=\max\limits_{r'} || X-X' ||_p
% \label{eq:max_x_error}
% \end{equation}

% Thus, we can describe the tuning process as an optimization problem formulated by \equref{eq:max_x_error_condition}.
% \begin{equation}
% r'=\max\limits_{r'} || X-X' ||_p, \\
% s.t. || r-r' ||_q  \leq \epsilon
% \label{eq:max_x_error_condition}
% \end{equation}

\begin{figure}[t]
\centering
     \includegraphics[scale=0.35]{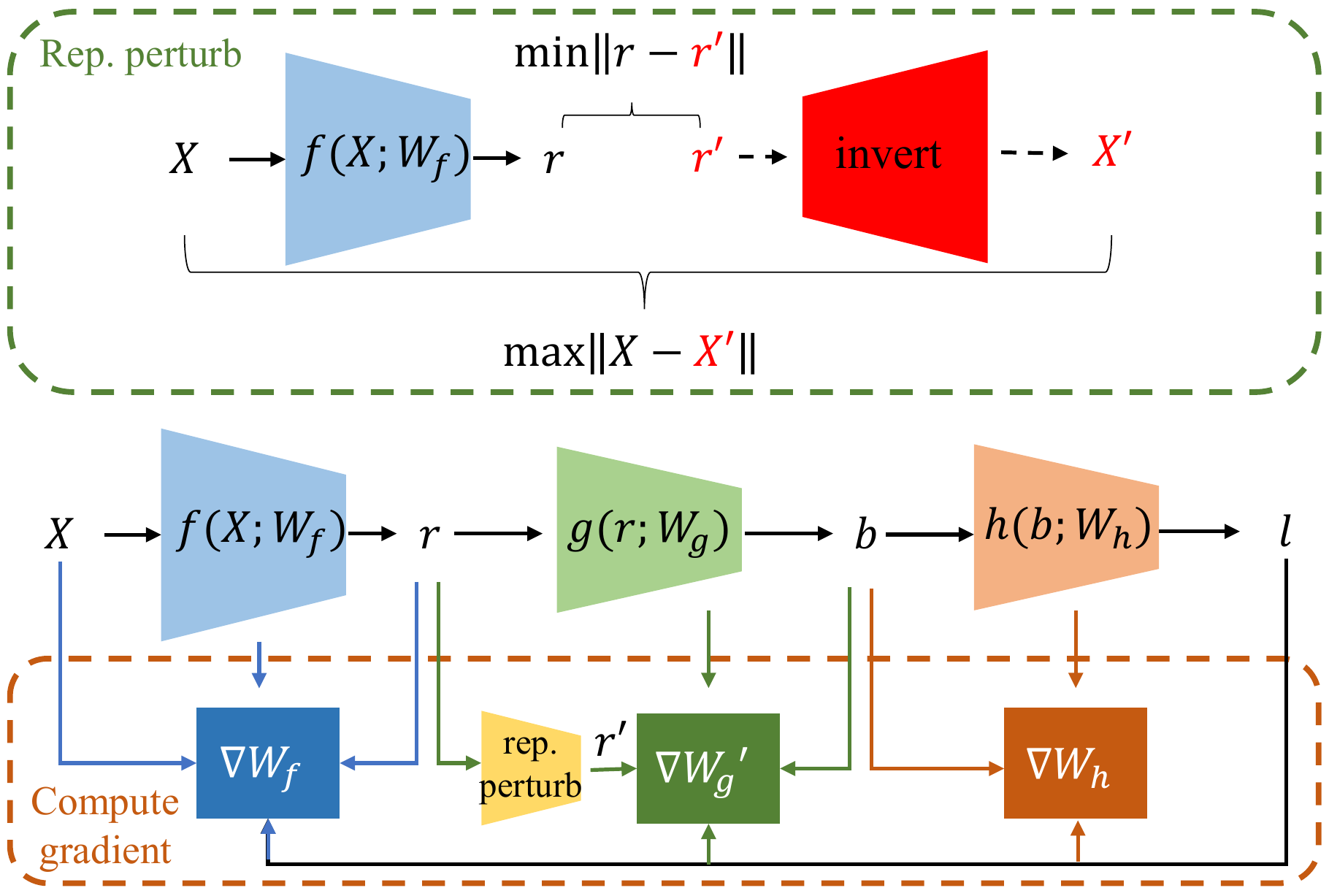}
\caption{Illustration of our representation perturbation defense.}
%{In our defense scheme, the data representation $r$ is tuned into $r'$ before being embedded into local updates. $\eta$ is local updating step size.}
\label{fig:defense_scheme}
\end{figure}

\subsection{Defense Solution}

Let $f:X\rightarrow{r}$ be the feature extractor
% extraction function 
before the defended layer. 
Prior to obtaining our solution, we make the following assumption and use the inverse function theorem. 
%For the feature extractor $f:X\rightarrow{r}$ in the local model, we make the following assumption.
\newtheorem{assumption}{Assumption}
\begin{assumption}\label{asm:f_inverse}
The inverse of $f$, i.e., $f^{-1}$, exists on $r$ and $r'$, $\forall || r-r' ||_q \leq \epsilon$. 
%, we have $X' = f^{-1}(r')$ and $|| f^{-1}(r)-f^{-1}(r') ||_p \approx || \nabla_{r} f^{-1} \cdot (r - r') ||_p$
\end{assumption}

% \begin{assumption}\label{asm:f_inverse}
% $\exists f^{-1}$, $\forall r' = f(X')$ satisfying $|| r-r' ||_q \leq \epsilon$, we have $X' = f^{-1}(r')$ and $|| f^{-1}(r)-f^{-1}(r') ||_p \approx || \nabla_{r} f^{-1} \cdot (r - r') ||_p$
% \end{assumption}

%Based on Assumption~\ref{asm:f_inverse}, we have the following inverse function theorem. 
\newtheorem{lemma}{Lemma}
\begin{lemma}\label{lemma:f_inverse}
For $\forall f:X\rightarrow{r}$ and $f^{-1}:r\rightarrow{X}$, $\nabla_r f^{-1} = {(\nabla_X f)}^{-1}$.
\end{lemma}

Then, our object function can be reduced as follows: 
%described by \equref{max_diff} and \equref{min_rep} can be formulated as follows: 
% \begin{align}
% r'&=\arg \max\limits_{r'} || X-X' ||_p  \label{eq:x_error_solution}\\ 
% &=\arg \max\limits_{r'} || f^{-1}(r)-f^{-1}(r') ||_p \label{eq:apply_inverse} \\ 
% &\approx \arg \max\limits_{r'} || \nabla_{r} f^{-1} \cdot (r - r') ||_p \label{eq:apply_asm1},  \\
% &= \arg\min \limits_{r'} || (r - r') \cdot \nabla_{X} f ||_p \label{eq:min_object}
% \end{align}
{\small
\begin{align}
r'&=\arg\max\limits_{r'} || X-X' ||_p , \, \,  s.t. || r-r' ||_q \leq \epsilon \label{eq:x_error_solution}\\ 
&=\arg\max\limits_{r'} || f^{-1}(r)-f^{-1}(r') ||_p, \, \, s.t. || r-r' ||_q \leq \epsilon \label{eq:apply_inverse} \\ 
&\approx \arg\max\limits_{r'} || \nabla_{r} f^{-1} \cdot (r - r') ||_p, \, \,  s.t. || r-r' ||_q \leq \epsilon \label{eq:apply_asm1}\\ 
&= \arg\max\limits_{r'} || {(\nabla_{X} f)}^{-1} \cdot (r - r') ||_p, \, \,  s.t. || r-r' ||_q \leq \epsilon \label{eq:min_object},
\end{align}
}%
where we use Assumption~\ref{asm:f_inverse} in Eq.~\ref{eq:apply_inverse}, use the first-order Taylor expansion in Eq.~\ref{eq:apply_asm1}, and use Lemma \ref{lemma:f_inverse} in Eq.~\ref{eq:min_object}.
% Viewing \equref{eq:min_object}, note that $\nabla_{X} f$ is known, then 
% our solution is to find the perturbed representation in the neighorbood of the clean representation, with a bounded $L_p$-norm $\epsilon$. Then following the direction of $r-r'$, the gradient of the representation to the input has the lowest magnitude.

% our perturbed representation 
% $\nabla_{X} f$ is known, and 
% Based on \equref{eq:min_object}, the perturbing process is to find a new representation in the neighbourhood of the real representation, which is bounded by $\epsilon$. Then following the direction of $r-r'$, the gradient of the representation to the input has the lowest magnitude.
Note that, with different choices of  $||.||_p$ and $||.||_q$, we have different defense solutions and thus have different defense effects. 
In this work, we set $p=2$, i.e., we aim to maximize the MSE between the reconstructed input and the raw input. 
Meanwhile, we set $q=0$ due to two reasons: our defense has an analytical solution 
% is analytical 
and is communication efficient. 
Specifically, our solution is to find the $\epsilon$ largest elements in the set $\{||r_i (\nabla_{X} f(r_i))^{-1} ||_2\}$. 
%(???? Should verify it), 
%which can be efficiently computed. 
Moreover, the learnt perturbed representation is relatively sparse and thus improves the communication efficiency. Algorithm~\ref{alg:train_r'} details the solution to obtain the perturbed presentation $r'$ with $q=0$ and $p=2$. Algorithm~\ref{alg:train_W'} details the local training process with our defense on a local device.
%It is noticeable that the detailed optimization process is related to the definition of $||.||_p$ and $||.||_q$, and different choices of $||.||_p$ and $||.||_q$ have different defense effects. In this work, we aim to maximize MSE between the reconstructed image and raw image, thus we define $||\cdot||_p$ as $||\cdot||_2$. We also define $||\cdot||_q$ as $||\cdot||_0$ for two reasons: optimization convenience and communication efficiency. The optimization process for defence object \equref{eq:min_object} is degraded to a pruning fashion when $||\cdot||_q$ in \equref{eq:min_object} is $||\cdot||_0$, and the optimized gradients will be sparse which improves the communication efficiency. 
%The detailed optimization process is described in Algorithm \ref{alg:train_r'}.

\begin{algorithm}[t] 
\footnotesize
\renewcommand{\algorithmicrequire}{\textbf{Input:}}
\renewcommand{\algorithmicensure}{\textbf{Output:}}
    \caption{Learning perturbed representation $r'$ with $q=0$ and $p=2$.}  
    \label{alg:Framwork}  
    \begin{algorithmic}[1] 
        \Require Training data $X \in \sR^{M\times N}$; Feature extractor $f:\sR^{M\times N}\rightarrow \sR^{L}$ before the defended layer; Clean data representation $r\in \sR^L$; Perturbation bound: $\epsilon $;
        \Ensure Perturbed data representation $r' \in \sR^L$;
        \Function {Perturb\_rep}{$X, f, r, \epsilon$}
            \State Compute $||r_i {(\nabla_{X} f(r_i))}^{-1}||_2$ for $i=0, 1, ..., L-1$;
            \State Find the set $\mathbb{S}$ which contains the indices of $\epsilon$ largest elements in $\{||r_i {(\nabla_{X} f(r_i))}^{-1}||_2\}_{i=1}^L$;
            \State $r' \leftarrow r$;
            \State Set $r'_i=0$ for $i\in \mathbb{S}$;
            \State \Return $r'$;
        \EndFunction
    \end{algorithmic} \label{alg:train_r'}
\end{algorithm} 
\setlength{\textfloatsep}{2mm}

% \begin{algorithm}[t]  
% \renewcommand{\algorithmicrequire}{\textbf{Input:}}
% \renewcommand{\algorithmicensure}{\textbf{Output:}}
%     \caption{Local training process of a FC layer with our defense.}  
%     \label{alg:Framwork}  
%     \begin{algorithmic}[1] 
%         \Require $\mX \in \sR^{M\times N}$, $f:\sR^{M\times N}\rightarrow \sR^{L}$, $\epsilon \in \sN$, $\mW\in \sR^{K\times L}$, $F:\sR^{M\times N}\rightarrow \sR$;
%         \Ensure $\mW' \in \sR^{K\times L}$; 
%         \Function {Update\_weight}{$\mX, f, F, \mW, \epsilon$}
%             \State $l \leftarrow F(\mX)$;
%             \State $r \leftarrow f(\mX)$;
%             \State $b \leftarrow \mW r$;
%             \State $r' \leftarrow Optimize\_representation(\mX, f, r, \epsilon)$;
%             \State $\mW' = \mW - \eta \frac{\partial l}{\partial b} r'^{T}$;
%             \State \Return $\mW'$;
%         \EndFunction
%     \end{algorithmic} \label{alg:train_W'}
% \end{algorithm} 

\begin{algorithm}[t]  
\footnotesize
\renewcommand{\algorithmicrequire}{\textbf{Input:}}
\renewcommand{\algorithmicensure}{\textbf{Output:}}
    \caption{Local training process with our defense on a local device. }
    \label{alg:Framwork}  
    \begin{algorithmic}[1] 
        \Require Training data $\mX \in \sR^{M\times N}$; Local objective function $F: \sR^{M\times N} \rightarrow \sR$; Feature extractor $f: \mW_f \in \sR^{M\times N}\rightarrow \sR^{L}$ before the defended layer; The defended layer $g: \mW_g \in \sR^{L}\rightarrow \sR^{K}$; Feature extractor after the defended layer $h: \mW_h \in  \sR^{K}\rightarrow \sR$; Local model parameters $\mW = \{\mW_f, \mW_g, \mW_h\}$; Learning rate $\eta$. 
        \Ensure Learnt model parameter $\mW$ with our defense.
        \State Initialize $\mW$;
        %\State $\mW_f \leftarrow {\mW_f}_0$; $\mW_g \leftarrow {\mW_g}_0$; $\mW_h \leftarrow {\mW_h}_0$;
        %\State $\mW \leftarrow \{\mW_f, \mW_g, \mW_h\}$;
        \For {$\sB$ in local training batches}
            \For {$\mX \in \sB$}
                \State $l\leftarrow F(\mX; \mW)$;
                \State $\vr\leftarrow f(\mX; \mW_f)$;
                \State $\vb\leftarrow g(\vr; \mW_g)$; // e.g., $\vb = \mW_{g} \vr$ for FC layers
                \State $l\leftarrow h(\vb; \mW_h)$;
                \State $\{\nabla \mW_f, \nabla \mW_g, \nabla \mW_h\}\leftarrow\nabla_{\mW}F(\mX;\mW)$; 
                \State $\vr'\leftarrow Perturb\_rep(\mX, f(;\mW_f), \vr, \epsilon)$;
                \State ${\color{red}\nabla \mW_g'} \leftarrow \tau(l, b, r', \mW_g)$; // e.g., $\nabla \mW_{g}' = \frac{\partial l}{\partial \vb} {\vr'}^{T}$ in FC 
                \State $\nabla \mW = \{\nabla \mW_f, {\color{red}\nabla \mW_g'}, \nabla \mW_h\}$;
                \State $\mW \leftarrow \mW - \eta \nabla{\mW}$;
            \EndFor
        \EndFor
        
    \end{algorithmic} \label{alg:train_W'}
\end{algorithm}

%\subsection{Defense Bound}
\subsection{Certified Robustness Guarantee}
We define our {\textit{certified robustness guarantee}} as the certified minimal distance (in terms of $L_p$-norm) between the raw input and the reconstructed input. 
A larger defense bound indicates that our defense is more effective. 
%the real data $X$ and the inverted data $X' = f^{-1}(r')$, which is formulated as $||X - X'||_{p}$. To calculate the defense bound of our defense scheme, we make the following assumption about $||\cdot||_{p}$.
Specifically, we have the following theorem on our defense bound: 
\newtheorem{theorem}{Theorem}
\begin{theorem}
\label{thm:bound}
Assuming Assumption~\ref{asm:f_inverse} holds. Given a data input $X$, its representation $r$ and any perturbed data representation $r'$, we have:
\begin{align}
||X - X'||_{p} \geq \frac{||r-r'||_{p}}{||\nabla_X f||_{p}}.
\end{align}
\end{theorem}
\begin{proof}
See our proof in Appendix~\ref{app:bound}.
\end{proof}

% As shown in \equref{eq:defense_lower_bound}, the lower bound of our defense scheme's {\textit{defense effectiveness}} is linearly related to $||r-r'||_{p}$, which is determined by the extend of the representation tuning process we conduct. The defense bound is also negatively correlated with $||\nabla_X f||_{p}$, which means that our defense scheme will be more effective if $f$ shows higher contractiveness.

%% file: theoretical_guarantee.tex
\section{Convergence Guarantee}
\label{sec:cvg_theory}

In this section, we derive the convergence guarantee of FedAvg~\cite{mcmahan2017communication}---the most popular FL algorithm, %particular the well-known 
 with our proposed defense. 
%In particular, we adopt FedAvg~\cite{} as the federated learning algorithm.
We first describe the FedAvg algorithm with our defense and then present our theorem on the convergence guarantee. 

\subsection{FedAvg with Our Defense}
In classical FedAvg, the objective function is defined as: 

{\small
\begin{equation}
\label{eqn:fedavg}
\mW = \min\limits_{\mW} \{ F(\mW) \triangleq \sum\limits_{k=1}^{N} p_{k} F_{k}(\mW) \},
\end{equation}
}%
where $p_k$ is the weight of the $k$-th device, $p_k \geq 0$ and $\sum_{k=1}^N p_{k} = 1$. 
$F_k$ is the local objective in the $k$-th device. 

Equation~\ref{eqn:fedavg} 
is solved via an iterative server-devices communication as follows: Suppose the server has learnt the global model ${\mW}_{t}$ in the $t$-th round, and randomly selects $K$ devices $\sS_t$ with replacement according to the sampling probabilities $p_1, ..., p_N$ for the next training round. Then FedAvg is performed as follows: 
First, the server sends the global model ${\mW}_{t}$ to all devices. 
Then, all devices set their local model to be ${\mW}_{t}$, i.e., ${\mW}_t^k = {\mW}_t, \forall k \in [1:N]$, and each device 
performs $I$ iterations of local updates. 
Specifically, for the $i$-th iteration, the local model in the $k$-th device applying our defense is updated as:
% For standard {\textit{FedAvg}} without our defense scheme, the local updates is
% \begin{equation}
% \begin{aligned}
% \mW_{t+i+1}^{k} \leftarrow \mW_{t+i}^{k} - {\eta}_{t+i}\nabla F_{k}(\mW_{t+i}^{k},{\xi}_{t+i}^{k}) \\
% , i = 0, 1, ... , I-1
% \end{aligned}
% \end{equation}
{\small
\begin{align}
%\begin{aligned}
& \nabla F'_{k}(\mW_{t+i}^{k},{\xi}_{t+i}^{k}) = \mathcal{T}(\nabla F_{k}(\mW_{t+i}^{k},{\xi}_{t+i}^{k}))\\
& \mW_{t+i+1}^{k} \leftarrow \mW_{t+i}^{k} - {\eta}_{t+i}\nabla F'_{k}(\mW_{t+i}^{k},{\xi}_{t+i}^{k}), 
%\end{aligned}
\end{align}
}%
where ${\eta}_{t+i}$ is the learning rate and ${\xi}_{t+i}^{k}$ is a data sample uniformly chosen from 
% the local data in 
the $k$-th device. $\mathcal{T}(\cdot)$ is our defense scheme.
Finally, the server averages the local models of the selected $K$ devices and updates the global model as follows:

{\small
\begin{equation}
\mW_{t+I} \leftarrow \frac{N}{K}\sum\limits_{k \in \sS_t} p_{k} \mW_{t+I}^{k}.
\end{equation}
}%

% \newtheorem{assumption_sample}{Assumption}
% \begin{assumption}\label{asm:sample}
% Assume $\sS_t$ contains a subset of $K$ indices randomly selected with replacement according to the sampling probabilities $p_1, ..., p_N$. The aggregation step of FedAvg performs $\mW_{t} \leftarrow \frac{1}{K}\sum\limits_{k \in \sS_t} p_{k} \mW_{t}$.
% \end{assumption}

\subsection{Convergence Analysis} 
Our convergence analysis is inspired by~\cite{li2019convergence}.  
Without loss of generality, we derive the convergence guarantee by applying our defense to a single layer. 
% (e.g., a FC layer). 
However, our results can be naturally generalized to multiple layers.  
% the case where we apply our defense to multiple FC layers. 
 We denote the input representation, parameters, and output of a single 
 % (the FC) 
(e.g. $s$-th) layer in the $k$-th device and in the $t$-th round as $r^k_t$, $\vw_{s}{_{t}^{k}}$ and $b^k_t$, respectively. 

%After applying our defense scheme $\mathcal{T}(\cdot)$, the input representation is perturbed as ${r'}^k_t$. 
Before presenting our theoretical results, we first make the following Assumptions~\ref{asm:L-smooth}-\ref{asm:gradient_norm} same as~\cite{li2019convergence} and an extra Assumption~\ref{asm:lamda} on bounding the squared norm of stochastic gradients with respect to the single $s$-th layer. 
%one fully connected layer $\vw_{s}$ 
%whose input real representation is $r$ and the output is $b$. 

\newtheorem{assumption_L-smooth}{Assumption}
\begin{assumption}\label{asm:L-smooth}
$F_1, F_2, ..., F_N$ are L-smooth: $\forall \mV, \mW$, $F_k(\mV) \leq F_k(\mW) + (\mV - \mW)^T \nabla F_k(\mW) + \frac{L}{2} ||\mV - \mW||_2^2$.
\end{assumption}

\newtheorem{assumption_u-convex}{Assumption}
\begin{assumption}\label{asm:u-convex}
$F_1, F_2, ..., F_N$ are $\mu$-strongly convex: $\forall \mV, \mW$, $F_k(\mV) \geq F_k(\mW) + (\mV - \mW)^T \nabla F_k(\mW) + \frac{\mu}{2} ||\mV - \mW||_2^2$.
\end{assumption}

\newtheorem{assumption_stochastic_gradients}{Assumption}
\begin{assumption}\label{asm:stochastic_gradients}
Let ${\xi}_{t}^{k}$ be sampled from the $k$-th device’s local data uniformly at random. The variance of stochastic gradients in each device is bounded: $\mathbb{E} ||\nabla F_{k}(\mW_{t}^{k},{\xi}_{t}^{k}) - \nabla F_{k}(\mW_{t}^{k}) ||^{2} \leq \sigma_k^2$ for $k = 1, ..., N$.
\end{assumption}

\newtheorem{assumption_gradient_norm}{Assumption}
\begin{assumption}\label{asm:gradient_norm}
The expected squared norm of stochastic gradients is uniformly bounded, i.e., $\mathbb{E} ||\nabla F_{k}(\mW_{t}^{k},{\xi}_{t}^{k}) ||^{2} \leq G^2$ for all $k = 1, ..., N$ and $t = 0, ..., T-1$.
\end{assumption}

\begin{assumption}\label{asm:lamda}
For the single $s$-th layer,
% fully connected layer, 
the squared norm of stochastic gradients on the output of each device is bounded: $||\nabla_{b_{t}^{k}} F_{k}({\vw_{s}}_{t}^{k},{\xi}_{t}^{k})||_2 \leq \Lambda_{s}$ for all $k = 1, ..., N$ and $t = 0, ..., T-1$.
\end{assumption}

We define $F^*$ and $F_k^*$ as the minimum value of $F$ and $F_k$ and let $\Gamma = F^* -  \sum\limits_{k=1}^{N} p_{k} F_{k}^{*}$. 
We assume each device has $I$ local updates and the total number of iterations is $T$. 
Then, we have the following convergence guarantee on FedAvg with our defense.  

\begin{theorem}
\label{thm:convergence}
Let Assumptions \ref{asm:L-smooth}-\ref{asm:lamda} hold and $L, \mu, \sigma_k, G, \Lambda_{s}$ be defined therein. Choose $\kappa = \frac{L}{\mu}$, $\gamma = \max \{8\kappa, I \}$ and the learning rate $\eta_t = \frac{2}{\mu(\gamma + t)}$. 
%Define $\Lambda_{\mathbb{FC}} = \sum\limits_{fc \in \mathbb{FC}} \Lambda_{s}$ where $\mathbb{FC}$ is the set of all fully connected layers in the model and 
Then FedAvg with our defense satisfies 
{\small
\begin{equation*}
\mathbb{E}[F(\mW_T)] - F^* \leq \frac{2\kappa}{\gamma + T}(\frac{Q+ C}{\mu} + \frac{\mu\gamma}{2} \mathbb{E}||\mW_0 - {\mW}^*||^2),
\end{equation*}
}%
where
{\small
\begin{align*}
%\begin{aligned}
& Q = \sum\limits_{k=1}^{N} p_k^2 (\Lambda_{s} \cdot \epsilon + \sigma_k^2) + 6L\Gamma + 8(I-1)^2(\Lambda_{s} \cdot \epsilon+G^2) \\
& C = \frac{4}{K}{I}^2(\Lambda_{s} \cdot \epsilon+G^2). 
%\end{aligned}
\end{align*}
}%
\end{theorem}

\begin{proof}
See our proof in Appendix~\ref{app:cvg_proof}.
\end{proof}

%% file: experiments.tex
\section{Experiments} \label{sec:experiments}

\subsection{Experimental Setup}
In our experiments, we evaluate our defense against two different model inversion attacks under non-IID settings. Experiments are conducted on a server with two Intel Xeon E5-2687W CPUs and four Nvidia TITAN RTX GPUs.

\paragraph{Attack methods.}
We evaluate our defense method against two model inversion attacks in FL.
\begin{itemize}
    \item \textbf{DLG attack~\cite{zhu2019deep}} assumes that a malicious server aims to reconstruct 
    % participating 
    devices' data using their uploaded gradients. In \textit{DLG} attack, the server optimizes reconstructed data  to minimize the Euclidean distance between the raw gradients and the gradients that are generated by the reconstructed data  in back propagation.
    \item \textbf{Gradient Similarity (GS) attack~\cite{geiping2020inverting}} shares the similar idea with \textit{DLG}. Different from using Euclidean distance in DLG, GS attack utilizes cosine similarity between the raw gradients and the dummy gradients to optimize the reconstructed data during local updates.
\end{itemize}

\paragraph{Defense baselines.}
% In our experiments, 
We compare our proposed defense with two existing defense methods: gradient compression (GC)~\cite{zhu2019deep} and differential privacy (DP)~\cite{mcmahan2018learning}. 

\begin{itemize}
    \item \textbf{GC} prunes gradients that are below a threshold magnitude, such that only a part of local updates will be communicated between devices and the server.
    \item \textbf{DP} protects privacy with theoretical guarantee by injecting noise to the gradients uploaded to the server. In the experiments, we separately apply Gaussian and Laplacian noise to develop two DP baselines, i.e., DP-Gaussian and DP-Laplace. 
\end{itemize}

\paragraph{Datasets.}
To evaluate our defense under more realistic FL settings, we use MNIST and CIFAR10 datasets and construct non-IID datasets by following the configurations in~\cite{mcmahan2017communication}. For each dataset, the data is distributed across 100 devices. Each device holds 2 random classes of data with 100 samples per class. 
%Specifically, we apply the \textit{DLG} attack to both the MNIST and CIFAR10 non-IID datasets, and apply the \textit{GS} attack to the CIFAR10 non-IID dataset. We do not evaluate the \textit{GS} attack on MNIST {\color{red} completely does not work??} 
By default, we perform training on CIFAR10 and MNIST non-IID dataset with 1000 and 200 communication rounds, respectively.
%, considering their different and on MNIST non-IID dataset for 200 communication rounds.

\paragraph{Hyperparameter configurations.}
In training, we set local epoch $E$ as 1 and batch size $B$ as 32. We apply SGD optimizer and set the learning rate $\eta$ to 0.01. 
In each communication round, there are 10 devices which are randomly sampled  to participate in the training. 
For model inversion attacks, the ideal case for the adversary is that there is only one sample in each batch, where the quality of reconstructed data will be very high~\cite{zhu2019deep}. We evaluate our defense in such an extreme case, but it should show much better performance in other general cases (i.e., more than one sample in each batch).
With regard to \textit{DLG} attack, we apply $L-BFGS$ optimizer and conduct 300 iterations of optimization to reconstruct the raw data. For \textit{GS} attack, we utilize Adam optimizer with a learning rate of 0.1 and report the reconstructed results after 120 iterations. The base model architectures for two attacks are presented in Appendix~\ref{sec:setup}.
For defense, the configurations of our method and the compared baselines are displayed in Table~\ref{tb:cor_average}, where $p_{model}$ in \textit{GC} stands for the pruning rate of the local models' gradients, $p_{fc}$ of our method represents the pruning rate of the the fully connected layer's gradients. Regarding   DP-Gaussian and DP-Laplace, we set the mean and variance of the noise distribution as 0 and $\sigma$, respectively. %stands for the variance of the noise.

\begin{table}[t]
    \centering
    \caption{Parameter configurations of different defense methods.}
        \begin{tabular}{l || c c}
            \toprule
            \textbf{Configured Parameters} & DLG & GS \\
            \hline
            GC: $p_{model}(\%)$ & [1, 80] & [1, 90] \\
            \hline
            DP-Gaussian: $\sigma_{Gaussian}$ & [$1e^{-4}$,$1e^{-1}$] & [$1e^{-4}$,$1e^{-1}$] \\
            \hline
            DP-Laplace: $\sigma_{Laplace}$ & [$1e^{-4}$,$1e^{-1}$] & [$1e^{-4}$,$1e^{-1}$] \\
            \hline
            Ours: $p_{fc}(\%)$ & [1, 40] & [1, 80] \\
            \bottomrule
        \end{tabular}
    \label{tb:cor_average}
\end{table}

\paragraph{Evaluation metrics.} % We consider both utility and privacy 
\begin{itemize}
\item {\bf Privacy metric (MSE):} We use the mean-square-error (MSE) between the reconstructed image and raw image to quantify the effectiveness of defenses. A smaller MSE indicates a server privacy information leakage.
\item {\bf Utility metric (Accuracy):} We use the accuracy of the global model on the testing set to measure the effectiveness of FL algorithms (i.e., FedAvg~\cite{mcmahan2017communication}). A smaller accuracy means a less practical utility. 
%For the accuracy of FL tasks, we report the accuracy of the learnt global model. 
\end{itemize}

\begin{figure*}[t]
\centering
     \includegraphics[scale=0.32]{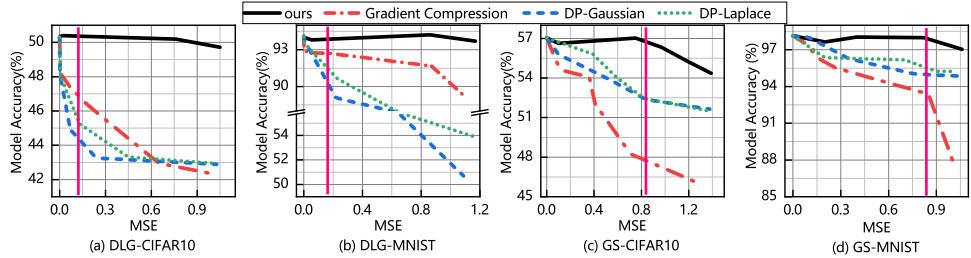}
\caption{Compared defenses on model accuracy and MSE between reconstructed image and raw image for different attack baselines and datasets. The {pink vertical line} is the boundary that the reconstructed image is unrecognizable by human eyes if MSE is higher.}
\label{fig:acc_mse}
\end{figure*}

\subsection{Defense Results: Utility-Privacy Tradeoff}

\begin{figure}[ht]
\centering
     \includegraphics[scale=0.35]{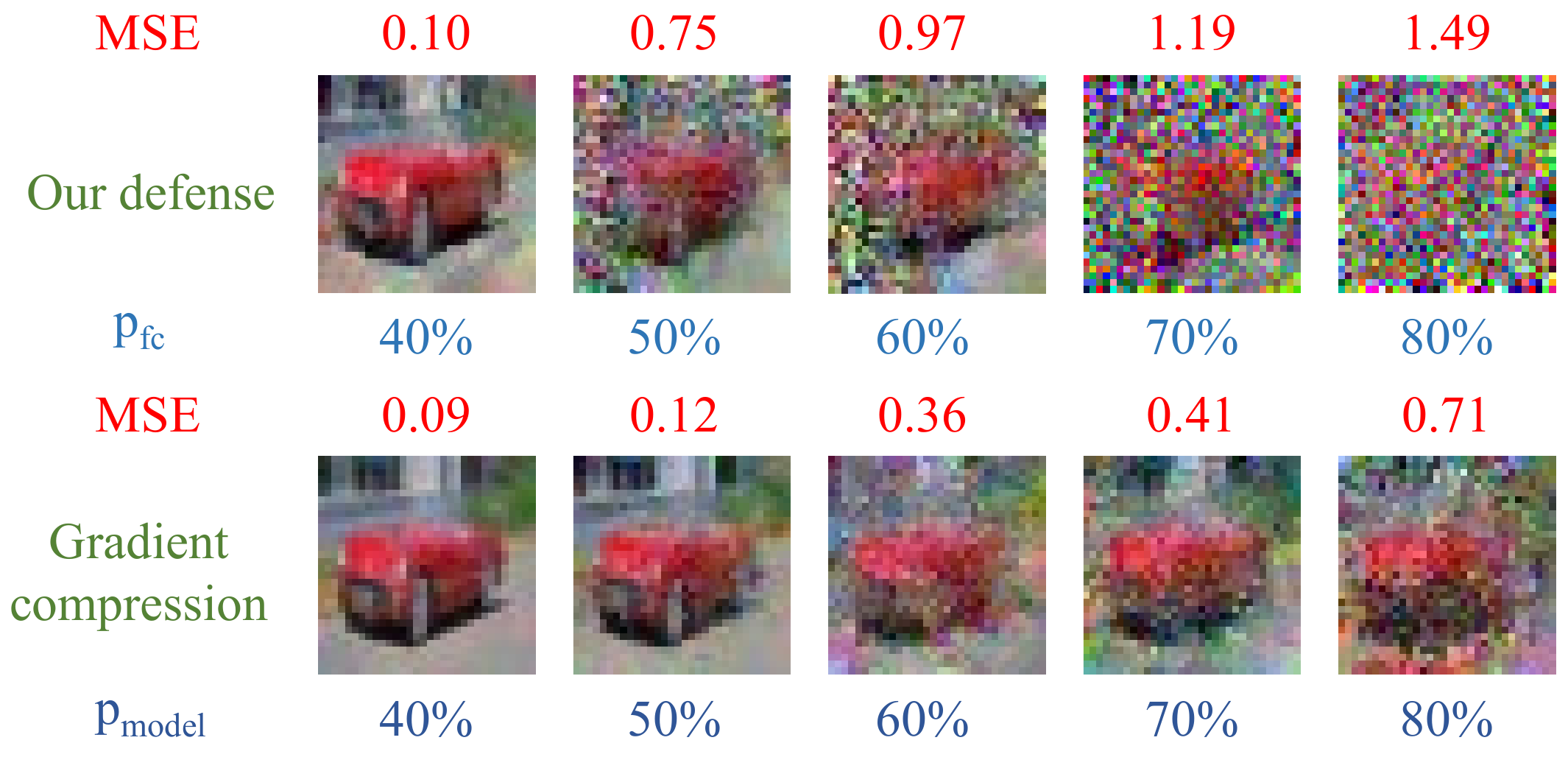}
\caption{Comparing our defense with GC under the same pruning rate on model accuracy and MSE on a random image in CIFAR10.}
\label{fig:GS_defense}
%\vspace{-4mm}
\end{figure}

We compare our defense with the baselines against the two attack methods in terms of model accuracy and MSE. Ideally, we want to maintain high model accuracy while achieving high MSE. The results are shown in \figref{fig:acc_mse}. 

We have the following two key observations. First, when achieving the MSE  such that the reconstructed image 
% cannot be 
is not recognizable by humans, our method shows no drop in accuracy while the other baselines sacrifice as high as $6\%$ and  $9\%$ accuracy under the DLG and GS attacks, 
% -CIFAR10 attack and GS-CIFAR10 attack, 
respectively. 

Second, without sacrificing accuracy, our defense can achieve 160x MSE than the baseline defenses. The accuracy can be maintained by our defense until MSE being 0.8, while the baselines show significant accuracy drop with a much smaller MSE. The reason is two folded: 1) our defense does not perturb parameters in the feature extractor (i.e., convolutional layers), which preserves the descriptive power of the model; and 2) the representations embedded in the gradients that are pruned by our defense are mostly inference-irrelevant, and hence pruning these parameters would be less harmful to the global model performance.

To perceptually demonstrate the effectiveness of our defense, we also visualize the reconstructed images. 
We compare our defense with GC, which is the defense baseline that also utilizes pruning. 
To save the space, we only show the results using the \textit{GS} attack but we have a similar observation in the \textit{DLG} attack. 
Figure~\ref{fig:GS_defense} shows the reconstructed image of a random sample in CIFAR10: 
the reconstructed image generated by our defense becomes unrecognizable when pruning only $50\%-60\%$ parameters in the FC layer. However, when applying the \textit{GC} defense, the reconstructed image is still recognizable even when $80\%$ parameters of the whole model are pruned. 
Note that being unrecognizable to humans is not the ultimate goal of defense, as the private information might still reside in the image though the image is not perceptually recognizable~\cite{ilyas2019adversarial}. 
Nonetheless, a MSE higher than the threshold that makes the image recognizable still serves as a meaningful indicator of privacy defense.

\subsection{Convergence Results}
Following the experimental setup in~\cite{li2019convergence}, we use a logistic regression (LR) to examine our convergence results on FedAvg using our defense. 
We distribute the MNIST dataset among $N=100$ devices in a non-IID setting where each device contains samples of 2 digits. 
% This is a stochastic convex optimization problem. 
Here, $\epsilon$ in Eq.~\ref{min_rep} is set to be 50, local batch size $B=32$, local epoch $E = \{5, 10\}$,  number of sampled devices $K$ in each communication round is selected from $\{5, 10\}$. 
%We conduct experiments on MNIST  non-IID dataset. 

Figure~\ref{fig:convergence} shows the results of loss vs. communication rounds. We observe that LR+FedAvg with our defense converges well, which validates our theoretical analysis.

\begin{figure}[t]
\centering
     \includegraphics[scale=0.32]{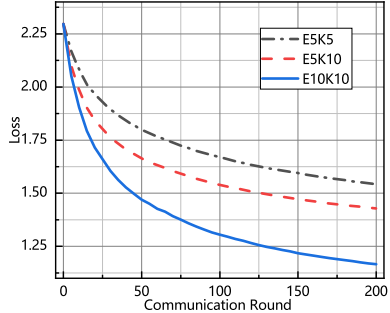}
\caption{Convergence of LR+FedAvg with our defense.}
\label{fig:convergence}
%\vspace{-4mm}
\end{figure}

%% file: conclusion.tex
\section{Conclusions and Future Work}
In this  work, we present our key observation that the data representation leakage from gradients is the essential cause of privacy leakage in FL. 
We also provide an analysis of this observation to explain how the data presentation is leaked. 
Based on this observation, we propose a defense against model inversion attack in FL. This is done by perturbing data representation such that the quality of the reconstructed data is severely degraded, while FL performance is maintained. In addition, we derive certified robustness guarantee to FL and convergence guarantee to FedAvg---the most popular FL algorithm, when applying our defense.
%we derive certified robustness guarantee to FL and convergence guarantee to FedAvg, when applying our defense. 
We conduct extensive experiments to evaluate the effectiveness of our defense, and the results demonstrate that our proposed defense can offer a much stronger privacy guarantee without sacrificing accuracy compared with baseline defenses.

Our further research include: 1) Investigating the impact of various $p$-norm and $q$-norm on both defense and accuracy, as well as designing norms that consider structural information in the data;  
%norm but they may not be limited to the $L$ norm. 
2) Extending our analysis of data representation leakage to other types of layers, e.g., convolutional layer, to have a more comprehensive understanding of privacy leakage in FL.

%% file: appendix_exp.tex
\input{method}

\section{Experiment Setup}\label{sec:setup}

\paragraph{Model for experiments in Section~\ref{sec:inferring representations}.}
For the inferring class-wise data representation experiment, we use the base model with 2 convolutional layers and 3 fully connected layers. The detailed architecture is listed as \textit{Conv3-6}$\rightarrow$\textit{Maxpool}$\rightarrow$\textit{Conv6-16}$\rightarrow$\textit{Maxpool}$\rightarrow$\textit{FC--120}$\rightarrow$\textit{FC--84}$\rightarrow$\textit{FC--10}. We set kernel size as 5 and 2 for all convolutional layers and max pooling layers respectively.

\paragraph{Settings for experiments in Section~\ref{sec:unveiling representation leakage}.}
For experiments unveiling representation leakage in Section~ \ref{sec:unveiling representation leakage}, we build a model with one convolutional layer and one fully connected layer. The detailed architecture is listed as \textit{Conv3-12}$\rightarrow$\textit{FC--10}. We set kernel size of the convolutional layer as 5. For attacks, we apply the \emph{L-BFGS} optimizer and conduct 300 iterations of optimization to reconstruct the raw data.

\paragraph{Models for two attacks in Section~\ref{sec:experiments}}
We use \textit{LeNet} for both the \textit{DLG} attack and \textit{ConvNet} for \textit{GS} attack. The architectures are shown in Table~\ref{tab:arch_attack}.

{\small
\linespread{0.85}
\begin{table}[ht]
    \centering
    \caption{Model architectures for \textit{DLG} attack and \textit{GS} attack.}
    \begin{tabular}{c||c}
    \toprule
      \textbf{\textit{DLG}} & \textbf{\textit{GS}} \\
      \hline
      5$\times$ 5 Conv 3-12   & 5$\times$ 5 Conv 3-32\\
      %\hline
      5$\times$ 5 Conv 12-12 & 5$\times$ 5 Conv 32-64 \\
      %\hline
      5$\times$ 5 Conv 12-12 & 5$\times$ 5 Conv 64-64\\
      %\hline
      5$\times$ 5 Conv 12-12 & 5$\times$ 5 Conv 64-128 \\
      %\hline
      FC--10 & 5$\times$ 5 Conv 128-128\\
      %\hline
      & 5$\times$ 5 Conv 128-128\\
      %\cline{2-2}
      & 3$\times$ 3 Maxpool\\
      %\cline{2-2}
      & 5$\times$ 5 Conv 128-128\\
      %\cline{2-2}
      & 5$\times$ 5 Conv 128-128\\
      %\cline{2-2}
      & 5$\times$ 5 Conv 128-128\\
      %\cline{2-2}
      & 3$\times$ 3 Maxpool\\
      %\cline{2-2}
      & FC--10\\
      %\cline{2-2}
      \bottomrule
    \end{tabular}

    \label{tab:arch_attack}
\end{table}
}

% \section{More Experimental Results}
% \label{app:more_exp}

% \begin{figure}[t]
%     \centering
%     \includegraphics[scale=0.32]{images/cov_fc123.eps}
%      \caption{$cor$ in different layers during communication rounds for CIFAR10. Under the same local training setting, the correlation between $\hat{\vr}_{\sB_i}^{\mathrm{T}}$ and $\overline{\vr}_{\sB_i}^{\mathrm{T}}$ is lower in shallower layers (layers closer to the input layer). This is because the data representations in the shallower layers are more general[], thus the structural specification of $Grad_{\sB}(i)$ in Equation \ref{eq:fc_grad_B} according to different data classes will be lower. As a result, the representation extraction in shallower layers will be harder compared to deeper layers.}
%     \label{fig:cifar10_fc}
% \end{figure}

% To study how $cor$ changes with training progress, we visualize the changes along the communication rounds.  As Figure \ref{fig:cifar10_fc} shows,  $cor$ drops  with the training  progress. Such a decrease results from the smaller magnitude of local updates with the training process moving forward. The insignificant updates means less data representations are embedded in model updates, leading to a harder representation extraction. However, \figref{fig:cifar10_fc} shows that $\hat{\vr}_{\sB_i}^{\mathrm{T}}$ still has more than 75\% correlation with $\overline{\vr}_{\sB_i}^{\mathrm{T}}$ in the last layer and the correlation tends to be  stable with the training process moving along, which means that representation leakage may incur at any time during FL training.

%% file: method.tex
\section{Method of Inferring Data Representations}\label{sec:method}

As discussed in Section~\ref{sec:leakage}, if we can find several rows in the local update $\mW$ that is from $Grad({\sB}_i)$, which is possible because of the low entanglement of $Grad({\sB}_i)$ across $i$ in FL, then we are able to infer this device's training data representation of class $i$ in this layer. As $\frac{\partial l^{j}}{\partial \vb^{j}}$ and $({\vr^{j}})^{\mathrm{T}}$ are both similar across $j$ in one batch $\sB_i$, $Grad({\sB}_i)$ can be approximated as \eqref{eq:gradB_approx}, 
{\small
\begin{equation}
\begin{aligned}
Grad({\sB}_i) & = \frac{1}{\left|\sB_i\right|} \sum_{j \in \sB_i} \frac{\partial l^{j}}{\partial \vb^{j}} {\vr^{j}}^{\mathrm{T}} \\
& \approx \left( \frac{1}{\left|\sB_i\right|} \sum_{j \in \sB_i} \frac{\partial l^{j}}{\partial \vb^{j}} \right) \left( \frac{1}{\left|\sB_i\right|} \sum_{j \in \sB_i} {\vr^{j}}^{\mathrm{T}} \right) \\
& = \overline{\frac{\partial l}{\partial \vb}}_{\sB_i} \overline{\vr}_{\sB_i}^{\mathrm{T}},
\end{aligned}
\label{eq:gradB_approx}
\end{equation}
}%
where $\overline{\frac{\partial l}{\partial \vb}}_{\sB_i}$ and $\overline{\vr}_{\sB_i}^{\mathrm{T}}$ denote the average of  $\frac{\partial l^j}{\partial \vb^j}$ and $({\vr^j})^{\mathrm{T}}$ for $j \in \sB_i$, and $\overline{\vr}_{\sB_i}^{\mathrm{T}}$ is the data representation corresponding to this device's training data of class $i$ in this layer. If we want to infer $\overline{\vr}_{\sB_i}^{\mathrm{T}}$ from this layer's local parameter update, we need to seek out the 
unique elements in $\overline{\frac{\partial l}{\partial \vb}}_{\sB_i}$. Here, unique elements are the elements in $\overline{\frac{\partial l}{\partial \vb}}_{\sB_i}$ that are not, or less entangled with other $\overline{\frac{\partial l}{\partial \vb}}_{\sB_{j, j \neq i}}$ after summation in \eqref{eq:fc_grad_B} is executed.

\subsection{Inferring features in the last layer}

Let us consider the last layer of a classification model with cross-entropy loss over a sample. Suppose $\vr$ is the data representation of the second-to-layer layer, we have 
%It is defined as 
%one-hot labels, which is formulated as
\begin{equation}
\begin{aligned}
\vb & = \mW \vr \\
\vy & = softmax(\vb)  \\
l & = - \log \vy_c,
\end{aligned}
\label{eq:last_layer}
\end{equation}
where $l$ is the loss defined on a sample and $c$ is the sample's  ground-truth label. 
$\vy = [ y_1, y_2, \dots, y_C]$ denotes the output of the $softmax$. %which contains the normalized confidences predicted for each class. 
Then $\frac{\partial l}{\partial \vb}$ in this layer is 
\begin{equation}
\begin{aligned}
{\frac{\partial l}{\partial \vb}}_i = 
\begin{cases}
y_i - 1, & i = c \\
y_i, & i \neq c 
\end{cases}
\end{aligned}
\label{eq:last_layer}
\end{equation}

As $y_1, y_2, \dots y_C$ are probabilities, we have $y_i \in (0, 1)$ and $\sum_i y_i = 1$. Hence, $\frac{\partial l}{\partial \vb}$ has only one negative element on index $c$ and the absolute value of $\frac{\partial l_c}{\partial \vb}$ is equal to the sum of other elements' absolute values. Therefore, for the last layer, the unique element in $\overline{\frac{\partial l}{\partial \vb}}_{\sB_i}$ is the "peak" element with index $i$, and this "peak" element contributes to the larger $||\nabla \mW_i||_2$, where $\nabla \mW_i$ denotes the $i^{th}$ row of $\nabla \mW$.

When the malicious server receives one local model updates, it computes $\{||\nabla \mW_1||_2,\allowbreak ||\nabla \mW_2||_2, \dots, ||\nabla \mW_C||_2\}$ and picks out the ones that are significantly larger. Then the server successfully infer data classes on this device because these selected rows' indexes corresponds to this device's training data classes. For one training class $i$, $\overline{\vr}_{\sB_i}^{\mathrm{T}}$ in this layer can just be approximated by $\gamma \nabla \mW_i$, where $\gamma$ is a scale influences by the local training steps. The algorithm of inferring data representations in the last layer is shown in Algorithm~\ref{alg:inferion_last_layer}.

\subsection{Inferring features in previous layers}

\begin{figure}[ht]
\centering
     \includegraphics[scale=0.5]{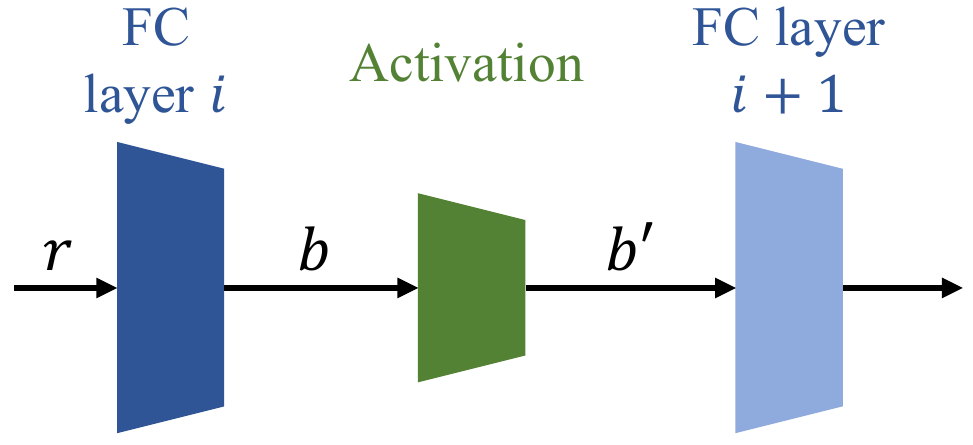}
\caption{Inference process.}
\label{fig:b_hat}
\end{figure}

Generally, we need to seek out the unique elements in $\overline{\frac{\partial l}{\partial \vb}}_{\sB_i}$ to infer $\overline{\vr}_{\sB_i}^{\mathrm{T}}$ in this layer. Let us assume we have inferred the data representation of $\sB_i$ in the layer after, which is denoted as $\overline{\vb'}_{\sB_i}$ shown in \figref{fig:b_hat}. Specifically, $\overline{\vb'}_{\sB_i}$ is the result of activation function with input as $\overline{\vb}_{\sB_i}$. If we can infer $\overline{\vr}_{\sB_i}^{\mathrm{T}}$ based on the access of $\overline{\vb'}_{\sB_i}$, plus the inferred last layer's data representation of $\sB_i$, then we can infer $\sB_i$'s data representations of every linear layer in a backpropagation fashion. 

Even though $\overline{\vb'}_{\sB_i}$ is a nonlinear transformation of $\overline{\vb}_{\sB_i}$, they share the similar structure and sparsity due to the consistency of most activation functions. Hence we can apply $\overline{\frac{\partial l}{\partial \vb'}}_{\sB_i}$ to approximate $\overline{\frac{\partial l}{\partial \vb}}_{\sB_i}$ for seeking the unique elements in $\overline{\frac{\partial l}{\partial \vb}}_{\sB_i}$. Theoretically, $\overline{\frac{\partial l}{\partial \vb'}}_{\sB_i}$ corresponds to the direction of $\nabla {\vb'}_{\sB_i}$. Because ${\vb'}_{\sB_i}$ should retain stable structure and sparsity in one local updating round as discussed in Section~\ref{sec:leakage}, $\nabla {\vb'}_{\sB_i}$ should mostly appear on the elements with larger magnitude. Therefore, the unique elements in $\overline{\frac{\partial l}{\partial \vb'}}_{\sB_i}$ should have the same indexes with the elements with larger magnitude in ${\vb'}_{\sB_i}$. Since we have access to ${\vb'}_{\sB_i}$, we can find $M$ most unique elements in $\overline{\frac{\partial l}{\partial \vb'}}_{\sB_i}$ by listing the $M$ elements in ${\vb'}_{\sB_i}$ with the largest magnitude. Then we can infer $\overline{\vr}_{\sB_i}^{\mathrm{T}}$ easily by fetching and averaging the rows of this layer's weight updates according to the $M$ unique elements indexes.

Following the above algorithm, the malicious server can fetch the training data representation in a fully connected layer for each data class on one device based on the data representation in the layer after. Plus the inference of all classes' training data representations in the last layer, the server is able to infer one device's training data representations for each class it owns in every fully connected layers in a back propagation way. The inferring process is shown in Algorithm~\ref{alg:inferion_previous_layer}.

\begin{algorithm}[t] 
\footnotesize
    \renewcommand{\algorithmicrequire}{\textbf{Input:}}
    \renewcommand{\algorithmicensure}{\textbf{Output:}}
    \caption{Data representation inference in the last layer.}  
    \label{alg:inferion_last_layer}  
    \begin{algorithmic}[1]
        \Require Local weight updates in the last layer $\nabla \mW$.
        \Ensure Local training class set $\sS$; Linearly scale training data representations $\{\hat{\vr}_{\sB_i}^{\mathrm{T}}, i\in \sS\}$ in this layer.
        \State Compute $||\nabla {\mW}_1||_2, ||\nabla {\mW}_2||_2, \dots, ||\nabla {\mW}_C||_2$;
        \State Pick up peaks of $\{||\nabla {\mW}_i||_2\}$ and collect their indexes as $\sS$;
        \State \Return $\sS$, $\{\nabla {\mW}_i, i\in \sS\}$;
    \end{algorithmic}
\end{algorithm}

\begin{algorithm}[t] 
\footnotesize
    \renewcommand{\algorithmicrequire}{\textbf{Input:}}
    \renewcommand{\algorithmicensure}{\textbf{Output:}}
    \caption{Data representation inference in previous fully connected layers.}  
    \label{alg:inferion_previous_layer}  
    \begin{algorithmic}[1]
        \Require Local weight updates in this layer $\nabla \mW$; Data representation $\overline{\vb'}_{\sB_i}$ for ${\sB_i}$ in the following layer; $M\in \sN+$.
        \Ensure Linearly scale training data representations $\hat{\vr}_{\sB_i}^{\mathrm{T}}$ in this layer.
        \State Select $M$ elements in $\overline{\vb'}_{\sB_i}$ with the largest magnitudes and collect their indexes as ${\mathbb{M}}$;
        \State $\hat{\vr}_{\sB_i}^{\mathrm{T}} \leftarrow \sum\limits_{k\in \mathbb{M}} \nabla {\mW}_k$;
        \State \Return $\hat{\vr}_{\sB_i}^{\mathrm{T}}$;
    \end{algorithmic}
\end{algorithm}

%% file: appendix_theorem.tex
\section{Proof of Theorem~\ref{thm:bound}}
\label{app:bound}

\newtheorem{proposition}{Proposition}
\newtheorem{assumption_p_norm}{Proposition}
\begin{proposition}
\label{asm:p_norm_compatible}
Let $||\cdot||_{p}$ be a sub-multiplicative norm. 
$||AB||_p \leq ||A||_p ||B||_p$. 
\end{proposition}

With Assumption \ref{asm:f_inverse} and 
Lemma \ref{lemma:f_inverse}, the distance between $X$ and $X'$ is:
\begin{equation}
\begin{aligned}
||X - X'||_{p} &= || f^{-1}(r)-f^{-1}(r') ||_{p} \\
&= ||\nabla_r f^{-1} \cdot (r-r')||_{p} \\
&= ||{(\nabla_X f)}^{-1} \cdot (r-r')||_{p}
\end{aligned}
\end{equation}

Based on Proposition~\ref{asm:p_norm_compatible}, 
we have $||C^{-1} D||_p \geq ||D||_p / ||C||_p$.
%, by replacing $||||$.
%$||AB||_p \geq ||A||_p ||B||_p$.
%Plus Assumption \ref{asm:p_norm_compatible}, 
Then, $||X - X'||_{p}$ is lower bounded as
\begin{equation}
\begin{aligned}
||X - X'||_{p} \geq \frac{||r-r'||_{p}}{||\nabla_X f||_{p}}.
\label{eq:defense_lower_bound}
\end{aligned}
\end{equation}

\section{Proof of Theorem~\ref{thm:convergence}}
\label{app:cvg_proof}

\noindent {\bf Overview:} Our proof is mainly inspired by~\cite{li2019convergence}. Specifically, our proof has two key parts. First, we derive the bounds similar to those in Assumptions~\ref{asm:stochastic_gradients} and~\ref{asm:gradient_norm}, after applying our defense scheme. 
Second, we adapt {\bf Theorem 2} on convergence guarantee in~\cite{li2019convergence} using our new bounds.   

\noindent {\bf Bounding the expected distance between the perturbed gradients with our defense and raw gradients using Assumption~\ref{asm:lamda}.}
In FedAvg, in the $t$-th round, we denote the input representation, parameters, and output of the single $s$-th 
% FC 
layer in the $k$-th device as $r^k_t$, $\vw_{s_t}^{k}$, and $b^k_t$, respectively.  
Via applying our defense scheme $\mathcal{T}(\cdot)$, the input representation is perturbed as ${r'}^k_t$.
%after applying our defense scheme $\mathcal{T}(\cdot)$ to the $k$-th, the representation $r$ is perturbed to $r'$. 
Then, the expected distance between the perturbed gradients and raw gradients in the $s$-th layer is bounded by:
\begin{align}
%\begin{aligned}
&\mathbb{E} ||\nabla F'_{k}({\vw_{s}}_{t}^{k},{\xi}_{t}^{k}) - \nabla F_{k}({\vw_{s}}_{t}^{k},{\xi}_{t}^{k})||_2 \\
=&\mathbb{E}||\nabla_{b_{t}^{k}} F_{k}({\vw_{s}}_{t}^{k},{\xi}_{t}^{k}) \cdot ({r'}_{t}^{k} - r_{t}^{k})^T ||_2 \\
\leq&\mathbb{E}||\nabla_{b_{t}^{k}} F_{k}({\vw_{s}}_{t}^{k},{\xi}_{t}^{k})||_2  \cdot ||({r'}_{t}^{k} - r_{t}^{k}) ||_2 \\
\leq& \Lambda_{s} \cdot \epsilon  \label{norm_assm_9},
%\end{aligned}
%\label{eqn:dist}
\end{align}
where in Eq.~\ref{norm_assm_9} we use the the constraint in Eq.~\ref{eq:min_object} by setting $q=2$ and Assumption \ref{asm:lamda}.
%By applying the constraint in \equref{eq:min_object} and Assumption \ref{asm:lamda}, 
%Considering the restriction in \equref{eq:min_object} and specializing $||\cdot||_q$ as $||\cdot||_2$, plus Assumption \ref{asm:lamda}, 
% we can bound the in \equref{eqn:dist} as follows:
% %distance between the tuned gradients and the raw gradients of $\vw_{s}$ as:
% \begin{equation}
% \begin{aligned}
% &||\nabla F'_{k}({\vw_{s}}_{t}^{k},{\xi}_{t}^{k}) - \nabla F_{k}({\vw_{s}}_{t}^{k},{\xi}_{t}^{k})||_2 \\
% \leq&||\nabla_{b_{t}^{k}} F_{k}({\vw_{s}}_{t}^{k},{\xi}_{t}^{k})||_2  \cdot ||({r'}_{t}^{k} - r_{t}^{k}) ||_2 \\
% \leq& \Lambda_{s} \cdot \epsilon
% \end{aligned}
% \label{eqn:bound_dist}
% \end{equation}

\noindent {\bf New bounds for Assumption~\ref{asm:stochastic_gradients} with our defense.}
Note that our defense scheme is only applied to the $s$-th 
% FC 
layer. 
Then, the distance between the perturbed gradients $\nabla F'_{k}({\mW}_{t}^{k},{\xi}_{t}^{k})$ and the raw gradients $\nabla F_{k}({\mW}_{t}^{k},{\xi}_{t}^{k})$  of the whole model is the same as that of the $s$-th
% FC 
layer. 
Thus, 
\begin{align}
%\begin{aligned}
&\mathbb{E} ||\nabla F'_{k}({\mW}_{t}^{k},{\xi}_{t}^{k}) - \nabla F_{k}({\mW}_{t}^{k},{\xi}_{t}^{k})||_2 \\
= & \mathbb{E}||\nabla F'_{k}({\vw_{s}}_{t}^{k},{\xi}_{t}^{k}) - \nabla F_{k}({\vw_{s}}_{t}^{k},{\xi}_{t}^{k}) ||_2 \\
\leq& \Lambda_{s} \cdot \epsilon. \label{eqn:bound_fc}
%=& \Lambda_{\mathbb{FC}} \cdot \epsilon
%\end{aligned}
\end{align}

%can be bounded by: 
% Because our defense tuning process $T(\cdot)$ only happens on fully connected layers, the distance between the tuned gradients $\nabla F'_{k}({\mW}_{t}^{k},{\xi}_{t}^{k})$ and the raw gradients of the whole model $\nabla F_{k}({\mW}_{t}^{k},{\xi}_{t}^{k})$ is composed of the corresponding distances of all the fully connected layers $\mathbb{FC}$, which can be bounded as:
% \begin{equation}
% \begin{aligned}
% &||\nabla F'_{k}({\mW}_{t}^{k},{\xi}_{t}^{k}) - \nabla F_{k}({\mW}_{t}^{k},{\xi}_{t}^{k})||_2 \\
% = & \sum\limits_{fc \in \mathbb{FC}}||\nabla F'_{k}({\vw_{s}}_{t}^{k},{\xi}_{t}^{k}) - \nabla F_{k}({\vw_{s}}_{t}^{k},{\xi}_{t}^{k}) ||_2 \\
% \leq& \sum_{fc \in \mathbb{FC}}\Lambda_{s} \cdot \epsilon \\
% =& \Lambda_{\mathbb{FC}} \cdot \epsilon
% \end{aligned}
% \end{equation}

% $\Lambda_{\mathbb{FC}}$ is determined by the model architecture. The more fully connected layers are, the larger $\Lambda_{\mathbb{FC}}$ is. And $\Lambda_{\mathbb{FC}}$ is larger when the fully connected layer's parameter $W_{fc}$ has more columns.

Next, we use the norm triangle inequality to bound he variance of stochastic gradients in each device, and we have 
\begin{align}
%\begin{aligned}
&\mathbb{E} ||\nabla F'_{k}(\mW_{t}^{k},{\xi}_{t}^{k}) - \nabla F_{k}(\mW_{t}^{k}) ||^{2} \\
\leq &\mathbb{E} ||\nabla F'_{k}(\mW_{t}^{k},{\xi}_{t}^{k}) - \nabla F_{k}(\mW_{t}^{k},{\xi}_{t}^{k}) ||^{2} \\
&+ \mathbb{E} ||\nabla F_{k}(\mW_{t}^{k},{\xi}_{t}^{k}) - \nabla F_{k}(\mW_{t}^{k}) ||^{2} \\
\leq& \Lambda_{s} \cdot \epsilon + \sigma_k^2 \label{eqn:assm_5}, 
%\end{aligned}
\end{align}
where we use Assumption \ref{asm:stochastic_gradients} and Eq.~\ref{eqn:bound_fc} in Eq.~\ref{eqn:assm_5}.
%which adapts Assumption \ref{asm:stochastic_gradients} to FedAvg utilizing our defense scheme. We can also adapt Assumption \ref{asm:gradient_norm} to the gradients after our defense tuning process as:

\noindent {\bf New bounds for Assumption~\ref{asm:gradient_norm} with our defense.}
%We adapt Assumption \ref{asm:gradient_norm} to bound 
The expected squared norm of stochastic gradients $\nabla F'_{k}(\mW_{t}^{k},{\xi}_{t}^{k})$ with our defense is as follows:
\begin{align}
%\begin{aligned}
&\mathbb{E} ||\nabla F'_{k}(\mW_{t}^{k},{\xi}_{t}^{k})||^{2} \\
\leq &\mathbb{E} ||\nabla F'_{k}(\mW_{t}^{k},{\xi}_{t}^{k}) - \nabla F_{k}(\mW_{t}^{k},{\xi}_{t}^{k}) ||^{2} \\
&+ \mathbb{E} ||\nabla F_{k}(\mW_{t}^{k},{\xi}_{t}^{k}) ||^{2} \\
\leq& \Lambda_{s} \cdot \epsilon + G^2, \label{eqn:assum_6}
%\end{aligned}
\end{align}
where we use Assumption~\ref{asm:gradient_norm} and and Eq.~\ref{eqn:bound_fc} in Eq.~\ref{eqn:assum_6}. 

\noindent {\bf Convergence guarantee for FedAvg with our defense.} 
We define $F^*$ and $F_k^*$ as the minimum value of $F$ and $F_k$ and let $\Gamma = F^* -  \sum\limits_{k=1}^{N} p_{k} F_{k}^{*}$. 
We assume each device has $I$ local updates and the total number of iterations is $T$. 
Let Assumptions \ref{asm:L-smooth} to \ref{asm:lamda} hold and $L, \mu, \sigma_k, G, \Lambda_{s}$ be defined therein. 
Choose $\kappa = \frac{L}{\mu}$, $\gamma = \max \{8\kappa, I \}$ , the learning rate $\eta_t = \frac{2}{\mu(\gamma + t)}$. 
By applying our new bounds and {\bf Theorem 2} in~\cite{li2019convergence},  
%Applying Lemma \ref{}, 
FedAvg using our defense has the following convergence guarantee:
{\small
\begin{equation}
\mathbb{E}[F(\mW_T)] - F^* \leq \frac{2\kappa}{\gamma + T}(\frac{Q+ C}{\mu} + \frac{\mu\gamma}{2} \mathbb{E}||\mW_0 - {\mW}^*||^2),
\end{equation}
}%
where
{\small
\begin{align*}
%\begin{aligned}
& Q = \sum\limits_{k=1}^{N} p_k^2 (\Lambda_{s} \cdot \epsilon + \sigma_k^2) + 6L\Gamma + 8(I-1)^2(\Lambda_{s} \cdot \epsilon+G^2) \\
& C = \frac{4}{K}{I}^2(\Lambda_{s} \cdot \epsilon+G^2). 
%\end{aligned}
\end{align*}
}%